\documentclass[a4paper,10pt]{scrartcl}

\usepackage[T1]{fontenc}
\usepackage[utf8]{inputenc}
\usepackage{lmodern}
\usepackage{amsfonts,amsmath,amssymb,amsthm,epsfig,graphicx,booktabs,latexsym,mathtools} %
\usepackage[margin=1.00in]{geometry}
\usepackage[english]{babel}

\usepackage{wrapfig}
\usepackage{mdframed}
\usepackage[authoryear,round]{natbib}
\usepackage{tikz}
\usetikzlibrary{positioning,shadows,arrows,shapes,shapes.arrows,arrows.meta,trees,shapes.misc,shapes.geometric,decorations.pathreplacing,calc}
\tikzset{
    -Latex,auto,node distance =1 cm and 1 cm,semithick,
    state/.style ={circle, draw=white, inner sep=0.01cm, minimum width = 0.35 cm},
    snode/.style = {rectangle, draw, inner sep=0.1cm,node contents={}},
    unmeasured/.style = {circle, draw, inner sep=0.05cm,node contents={}},
    point/.style = {circle, draw, inner sep=0.04cm,fill,node contents={}},
    bidirected/.style={Latex-Latex,dashed},
    el/.style = {inner sep=2pt, align=left, sloped}
}
\tikzstyle{mybox} = [draw=gray, fill=gray!20, very thick,
rectangle, rounded corners, inner xsep=3pt, inner ysep=7pt]
\usepackage{bm}

\usepackage[on]{settings/editing}
\usepackage{enumerate}
\usepackage{cancel}
\usepackage{multirow}
\usepackage{color}
\usepackage{authblk,epstopdf}
\usepackage{array} 
\usepackage{paralist}
\usepackage{verbatim} 
\usepackage[normalem]{ulem}
\usepackage{thmtools}
\usepackage[T1]{fontenc}
\usepackage[utf8]{inputenc}
\usepackage{lmodern}
\usepackage{setspace}
\usepackage{adjustbox}
\usepackage{graphicx}
\usepackage{subfig}
\usepackage[useregional]{datetime2}
\usepackage{clipboard}
\usepackage{enumitem,kantlipsum} 
\usepackage[some]{background}
\usepackage{optidef}
\usepackage[most]{tcolorbox}
\tcbuselibrary{theorems}
\usepackage{longtable}
\usepackage{ltablex}
\usepackage{nccmath}
\usepackage{pifont}
\usepackage{sansmath}
\usepackage{booktabs,tabularx}    
\usepackage{makecell}  

\newcommand{\cmark}{\ding{51}}%
\newcommand{\xmark}{\ding{55}}%

\definecolor{betterred}{RGB}{228,26,28}
\definecolor{betterblue}{RGB}{55,126,184}

\newcommand\labelAndRemember[2]
  {\expandafter\gdef\csname labeled:#1\endcsname{#2}%
   \label{#1}#2}
\newcommand\recallLabel[1]
   {\csname labeled:#1\endcsname\tag{\ref{#1}}}

\let\oldnl\nl
\newcommand{\nonl}{\renewcommand{\nl}{\let\nl\oldnl}}
\definecolor{opaquegray}{RGB}{192, 192, 192}
\definecolor{afblue}{RGB}{0,0,102}
\definecolor{mypink}{RGB}{255,51,153}

\newtheorem{definition}{Definition}
\newtheorem{theorem}{Theorem}
\newtheorem{assumption}{Assumption}

\newtheorem{lemma}{Lemma}

\newtheorem{proposition}{Proposition}
\newtheorem{corollary}{Corollary}

\tcolorboxenvironment{definition}{
  colback=gray!5!white,
  colframe=black,
  fonttitle=\bfseries,
  boxrule=0.5pt, 
  breakable,
}

\tcolorboxenvironment{theorem}{
  colback=gray!5!white,
  colframe=black,
  fonttitle=\bfseries,
  boxrule=0.5pt, 
  breakable,
}

\tcolorboxenvironment{corollary}{
  colback=gray!5!white,
  colframe=black,
  fonttitle=\bfseries,
  boxrule=0.5pt, 
  breakable,
}

\tcolorboxenvironment{lemma}{
  colback=gray!5!white,
  colframe=black,
  fonttitle=\bfseries,
  boxrule=0.5pt, 
  breakable,
}

\tcolorboxenvironment{proposition}{
  colback=gray!5!white,
  colframe=black,
  fonttitle=\bfseries,
  boxrule=0.5pt, 
  breakable,
}

\tcolorboxenvironment{proof}{
  colback=white,
  colframe=black,
  fonttitle=\bfseries,
  boxrule=0.5pt, 
  breakable,
}

\tcolorboxenvironment{example}{
  colback=white,
  colframe=black,
  fonttitle=\bfseries,
  boxrule=0.5pt, 
  breakable,
}

\tcolorboxenvironment{remark}{
  colback=gray!5!white,
  colframe=black,
  fonttitle=\bfseries,
  boxrule=0.5pt, 
  breakable,
}

\tcolorboxenvironment{assumption}{
  colback=white,
  colframe=black,
  fonttitle=\bfseries,
  boxrule=0.5pt, 
  breakable,
}

\usepackage[dvipsnames]{xcolor}
\usepackage[colorlinks=true, citecolor=afblue, linkcolor=RoyalBlue]{hyperref}%

\title{\Copy{title}{\huge{Debiased Front-Door Learners for Heterogeneous Effects}}}

\author{ \textbf{Yonghan Jung} 
\\ \normalsize{University of Illinois Urbana-Champaign}
\\ \normalsize{ \texttt{yonghan@illinois.edu} } 
}

\date{}
\begin{document}
\maketitle
\allowdisplaybreaks


\begin{abstract}
    \begin{center}
        \textbf{Abstract}
    \end{center}
    \medskip 
     \noindent In observational settings where treatment and outcome share unmeasured confounders but an observed mediator remains unconfounded, the front-door (FD) adjustment identifies causal effects through the mediator. We study the \textit{heterogeneous treatment effect} (HTE) under FD identification and introduce two debiased learners: \textit{FD-DR-Learner} and \textit{FD-R-Learner}. Under explicit sample-splitting, bounded-overlap, moment, and stage-learning assumptions, we show that FD-DR satisfies a product-error bound and FD-R satisfies a stage-error decomposition; these results yield conditional quasi-oracle corollaries when the relevant nuisance remainders are no larger than the target or stage oracle terms. We provide error analyses establishing this debiasedness and demonstrate robust empirical performance in synthetic studies and a real-world case study of primary seat-belt laws using Fatality Analysis Reporting System (FARS) dataset. Together, these results indicate that the proposed learners can deliver reliable and sample-efficient HTE estimates in FD scenarios when the stated assumptions are credible. The implementation is available at \url{https://github.com/yonghanjung/FD-CATE}.

\medskip
\noindent\textbf{Keywords:} Front-door adjustment, Heterogeneous treatment effects, Debiased learning, Conditional quasi-oracle guarantees, DR-Learner, R-Learner

\end{abstract}

\tableofcontents
\newpage

\begingroup
\allowdisplaybreaks
\sloppy
\section{Introduction}

Estimating causal effects from observational data is central to disciplines such as public policy. A crucial challenge is unmeasured confounding, where the treatment a unit receives is influenced by unobserved variables that also affect the outcome. The \textit{front-door} (FD) criterion \citep{pearl:95} addresses this by using an observed mediator that transmits the treatment’s influence to the outcome and is plausibly explained by observed covariates.

A concrete instance (which we also use in our empirical study in \S~\ref{subsec:real-data}) is the effect of adopting a primary seat-belt law ($X$) on occupant fatality ($Y$), with observed seat-belt use ($Z$) as the mediator, and observed covariate $C$. This scenario is depicted in Fig.~\ref{fig:FD}. Because of unmeasured confounding between $X$ and $Y$, the naive contrast $\mathbb{E}[Y \mid X=1] - \mathbb{E}[Y \mid X=0]$ is a biased estimate of the causal effect. When belt use responds strongly to law adoption and the observed covariates plausibly explain belt use, the setting is consistent with an FD structure, enabling identification and estimation of the causal effect via FD adjustment \citep{pearl:95}.

Although robust FD estimation has advanced recently (refer \S~\ref{subsec:related-works}), most methods target population averages (the \emph{average treatment effect}). In practice, platforms require \emph{personalized} effects—i.e., the conditional front-door effect $\tau(C)$—to support decision-making at the user or context level.
We address this gap by adapting heterogeneous treatment effect estimators developed for the standard ignorability \citep{rubin1974Causal} (or back-door adjustment \citep{pearl:95}), such as the DR-Learner \citep{kennedy2020optimal} and the R-learner \citep{nie2021quasi}, to the front-door setting. Concretely,
\begin{enumerate}[leftmargin=*]
    \item \textbf{FD-DR-Learner.} We construct a front-door pseudo-outcome whose conditional mean equals $\tau(C)$ at the true nuisances. With estimated nuisances, its conditional bias is controlled by product-error terms, so the final regression can inherit the target learner's oracle behavior when these remainders are sufficiently small.
    \item \textbf{FD-R-Learner.} We develop a three-stage procedure: (1) estimate how $X$ changes $Z$ across $C$; (2) estimate how changes in $Z$ shift $Y$ given $(X,C)$; (3) compose these pathway estimates to obtain $\tau(C)$. Its guarantee is stated through explicit stage-error bounds for the pathway learners and the final regression.
    \item \textbf{Theory and Practice.} We provide assumption-forward error analyses of the proposed estimators, showing conditional quasi-oracle corollaries when bounded overlap, sample splitting, moment bounds, and stage-learning guarantees hold. We demonstrate our findings with synthetic and real-world data analysis.
\end{enumerate}
This paper is organized as follows. \S\ref{sec:fd-dr-learner}-\ref{sec:fd-r-learner} presents the FD-DR and FD-R learners. \S\ref{sec:experiments} reports simulations and a case study. All proofs are deferred to Section~\ref{sec:proofs}.

\subsection{Related works}\label{subsec:related-works}
For back-door adjustment (g-formula), a large literature establishes debiased and sample-efficient estimation. Classical approaches include augmented inverse probability weighting \citep{robins1994estimation,bang2005doubly}, as well as targeted maximum likelihood estimation (TMLE) \citep{van2006targeted,van2012targeted}. More recently, estimation frameworks leveraging machine learning for heterogeneous treatment effect have produced flexible estimators with finite-sample guarantees, notably the DR-Learner \citep{kennedy2020optimal}, the R-learner \citep{nie2021quasi}, the EP-learner \citep{van2024combining}, orthogonal statistical learning \citep{foster2019orthogonal}, and double/debiased machine learning \citep{chernozhukov2018double}.

Beyond back-door adjustment, research on developing debiased estimators for the front-door (FD) model has grown steadily. \citet{fulcher2017robust} develop a doubly robust estimator for the FD average treatment effect estimator. \citet{guo2023targeted} propose one–step and TMLE estimators for FD adjustment.  \citet{jung2024uca} introduce a unified covariate-adjustment formulation that enables robust and sample-efficient FD estimation. On a different thread, modern deep learning-based FD estimators have been developed for scalable estimation \citep{xu2022neural,xu2025FD}, but these methods lack the debiasedness property.

Beyond average effects, work on \emph{heterogeneous} or \emph{conditional} FD has also emerged. \citet{chen2025conditional} study conditional FD and introduce \emph{LobsterNet}, a multi-task neural estimator for the conditional FD effect. However, their estimator lacks debiasedness. In this work, we propose debiased learners for heterogeneous FD treatment effects, which allow flexible machine learning methods to be used while giving conditional rate guarantees when the relevant product-error and stage-error remainders are controlled. A summary comparing existing works with ours is in Table~\ref{tbl:comparison}.

\begin{figure}[t]
    \hfill
    \begin{minipage}[b]{0.25\textwidth}\centering
        \subfloat[]{
            \begin{tikzpicture}[x=1.2cm,y=.9cm,>={Latex[width=1.4mm,length=1.7mm]},
  font=\sffamily\sansmath\scriptsize,
  RR/.style={draw,circle,inner sep=0mm, minimum size=4.5mm,font=\sffamily\sansmath\footnotesize}]
        \pgfmathsetmacro{\u}{1.0}
        \node[RR, betterblue] (X) at (0*\u,0*\u){$X$};
        \node[RR] (Z) at (1*\u,0*\u){$Z$};
        \node[RR] (C) at (1*\u,1.5*\u){$C$};
        \node[RR, betterred] (Y) at (2*\u,0*\u){$Y$};
        \node[RR, gray!70] (U) at (-0.75*\u, 1.5*\u){$U$};

        \draw[->] (C) to [bend right=0] (X);
        \draw[->] (C) to [bend right=0] (Z);
        \draw[->] (C) to [bend right=0] (Y);
        \draw[->] (X) to [bend right=0] (Z);
        \draw[->] (Z) to [bend right=0] (Y);

        \begin{scope}[gray!70]
            \draw[->] (U) to [bend right=0] (X);
            \draw[->] (U) to [bend right=0] (C);
            \draw[->] (U) to [bend right=-60] (Y);
        \end{scope}
    \end{tikzpicture}\label{fig:FD}}
    \end{minipage}
    \hfill
    \begin{minipage}[b]{0.25\textwidth}\centering
        \subfloat[]{
            \begin{tikzpicture}[x=1.2cm,y=.9cm,>={Latex[width=1.4mm,length=1.7mm]},
  font=\sffamily\sansmath\scriptsize,
  RR/.style={draw,circle,inner sep=0mm, minimum size=4.5mm,font=\sffamily\sansmath\footnotesize}]
        \pgfmathsetmacro{\u}{1.0}
        \node[RR, betterblue] (X) at (0*\u,0*\u){$X$};
        \node[RR] (C) at (1*\u,1.5*\u){$C$};
        \node[RR, betterred] (Y) at (2*\u,0*\u){$Y$};

        \draw[->] (C) to [bend right=0] (X);
        \draw[->] (C) to [bend right=0] (Y);
        \draw[->] (X) to [bend right=0] (Y);

    \end{tikzpicture}\label{fig:BD}}
    \end{minipage}
    \hfill
    \begin{minipage}[t]{0.45\linewidth}\centering
      \small
      \setlength{\tabcolsep}{3pt}
      \vspace{-100pt}
      \subfloat[]{%
        \renewcommand{\arraystretch}{1.05}%
        \begin{tabular}{@{}lccc@{}}
          \toprule
          & FD & HTE & Debiasedness \\
          \midrule
          \makecell[l]{\cite{nie2021quasi}\\ \cite{kennedy2020optimal}}
            & \textcolor{red}{\xmark} & \textcolor{blue}{\cmark} & \textcolor{blue}{\cmark} \\
          \midrule
        \makecell[l]{\cite{xu2022neural}\\ \cite{xu2025FD}}  & \textcolor{blue}{\cmark} & \textcolor{red}{\xmark} & \textcolor{red}{\xmark} \\
            \midrule
          \makecell[l]{\cite{fulcher2017robust}\\ \cite{guo2023targeted} \\ \cite{wen2024causal}   \\ \cite{jung2024uca} }    & \textcolor{blue}{\cmark} & \textcolor{red}{\xmark} & \textcolor{blue}{\cmark} \\
          \midrule
          \makecell[l]{ \cite{chen2025conditional}
          }
            & \textcolor{blue}{\cmark} & \textcolor{blue}{\cmark} & \textcolor{red}{\xmark} \\
          \midrule\midrule
          \textbf{Ours}
            & \textcolor{blue}{\cmark} & \textcolor{blue}{\cmark} & \textcolor{blue}{\cmark} \\
          \bottomrule
        \end{tabular} \label{tbl:comparison}%
      }
    \end{minipage}
    \caption{\textbf{(a)} Causal diagram illustrating the front-door (FD) structure; \textbf{(b)} back-door (BD) structure; \textbf{(c)} comparison table indicating whether methods estimate FD effects, heterogeneous treatment effects (HTE), and are debiased. Our proposed learners satisfy all three.}
\end{figure}
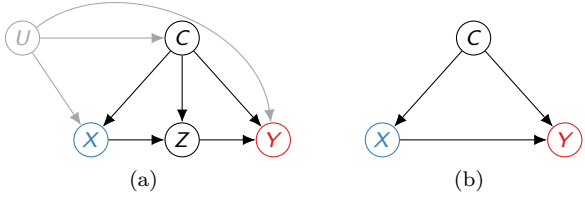

\section{Problem Setup}\label{sec:problem-setup}
We observe i.i.d. samples of $V=(C,X,Z,Y)$, where $C$ is observed covariates, $X \in \{0,1\}$ a binary treatment, $Z \in \{0,1\}$ a binary mediator, and $Y$ is the outcome. Throughout, concatenated arguments denote tuples in the displayed order: for example, $xC=(x,C)$, $zC=(z,C)$, $zxC=(z,x,C)$, $XC=(X,C)$, $ZC=(Z,C)$, $ZXC=(Z,X,C)$, and $YXC=(Y,X,C)$. We also write $e_x(C)\equiv e(x\mid C)$ for $x\in\{0,1\}$, and all sums over $x$ and $z$ are over $\{0,1\}$. Let $U$ represent an unobserved variable that jointly influences $(C,X,Y)$. We define the nuisance functions
\begin{align}
    m(zxc) &\triangleq \mathbb E[Y\mid Z=z, X=x, C=c], \label{eq:nuisance-outcome} \\
    e(x \mid c) &\triangleq \Pr(X=x \mid C=c) \quad \text{ with } e_1(C) \triangleq e(1 \mid C), \text{ and } e_0(C) \triangleq e(0 \mid C). \label{eq:nuisance-propensity} \\
    q(z \mid xc) &\triangleq \Pr(Z=z \mid X=x,C=c) \label{eq:nuisance-mediator-propensity}
\end{align}
We assume positivity in the identification statement: $0 <e(x \mid c), q(z \mid xc) <1$ for $x,z \in \{0,1\}$ and almost every $c$. We also assume $\mathrm{Var}(Y)<\infty$.

The data-generating process is depicted in the causal diagram $\mathcal{G}$ in Fig.~\ref{fig:FD}. Specifically, the structure satisfies the conditional front-door criterion \citep{pearl:95, fulcher2017robust}:
\begin{enumerate}[leftmargin=*]
    \item Every directed path from $X$ to $Y$ is mediated by $Z$ in $\mathcal{G}$ (no unmediated direct effect).
    \item Every spurious path between $X$ and $Z$ is blocked (d-separated) by $C$ in $\mathcal{G}$.
    \item Every spurious path between $Z$ and $Y$ is blocked (d-separated) by $(X,C)$ in $\mathcal{G}$.
\end{enumerate}
Let $\tau_{\bar{x}}(C) \triangleq \mathbb{E}\!\big[Y \mid \mathrm{do}(X=\bar{x}),\,C\big]$, where $\mathrm{do}(X=\bar{x})$ denotes an intervention that fixes $X$ to $\bar{x}$. Under the graph in Fig.~\ref{fig:FD}, the conditional treatment effect $\tau(C)$ is identified by
\begin{align}\hypertarget{eq:target-estimator}{}\label{eq:target-estimator}
    \tau(C) &\triangleq \tau_1(C) - \tau_0(C)  = \sum_{z,x}\{q(z \mid 1C) - q(z \mid 0C)\}e_x(C) m(zxC),
\end{align}

\paragraph{Assumptions for theoretical guarantees.}
The finite-sample and excess-risk statements below use the following strengthened conditions. These assumptions are stated once here and are invoked in Thms.~\ref{thm:error-fd-dr} and~\ref{thm:error-fd-r}.
\begin{enumerate}[label=(A\arabic*),leftmargin=*]
    \item The observed sample is i.i.d.; nuisance and target stages use sample splitting or cross-fitting so that each target regression is evaluated on observations independent of the nuisance fits used in its pseudo-outcome.
    \item The treatment $X$ and mediator $Z$ are binary, and the conditional front-door conditions above hold with the standard consistency and intervention semantics.
    \item There exists $c_0>0$ such that $c_0\le e(x\mid C)\le 1-c_0$ and $c_0\le q(z\mid xC)\le 1-c_0$ almost surely for all $x,z$. Estimated denominators used in inverse weights or density ratios satisfy the same lower bound on the event analyzed.
    \item The conditional means, pseudo-outcomes, and nuisance functions used below have the finite moments required by the displayed $L_2$ and $L_4$ norms; whenever a bound uses a sup-norm factor, that factor is finite on the same event.
    \item The target and stage learners satisfy the excess-risk or stage-error inequalities explicitly assumed in each theorem. These guarantees may come from empirical-process arguments, algorithm-specific oracle inequalities, or the cited R-learner interface.
\end{enumerate}

\subsection{Preliminaries: R-Learner for Back-Door Adjustment (BD-R-Learner)}

In this subsection, we review the standard R-learner of \citep{nie2021quasi}, a key building block for developing our FD-R-Learner. We refer to this standard R-learner as \textit{BD-R-Learner} (shortly, BDR), since it is developed under the back-door (BD) criterion \citep{pearl:95}, as depicted in Fig.~\ref{fig:BD}.

To review the BD-R-Learner, define the nuisance functions $\eta \triangleq \{e_X, m_Y\}$, where
\begin{align}
    e_X(C) \triangleq e(1 \mid C) \quad \text{ and } \quad m_Y(C) \triangleq \mathbb{E}[Y \mid C].
\end{align}
Let $\tau^{\mathrm{BD}}_{\bar{x}}(C) \!\triangleq\! \mathbb{E}[Y \!\mid\! \mathrm{do}(X=\bar{x}),C]$ in the BD setting. It is identified as
\begin{align}
    \tau^{\mathrm{BD}}(C) \triangleq \tau^{\mathrm{BD}}_1(C) - \tau^{\mathrm{BD}}_0(C) \!=\! \mathbb{E}[Y \!\mid\! X=1,C] \!-\! \mathbb{E}[Y \!\mid\! X=0,C].
\end{align}
In the BD graph in Fig.~\ref{fig:BD}, the data generating process can be expressed as the following partial linear model \citep{robinson1988root}:
\begin{align}\label{eq:robinson}
    X &= e_X(C) + \epsilon_X, \quad \mathbb{E}[\epsilon_X \mid C] = 0; \quad \text{ and } \\
    Y &= a(C) + Xb(C) + \epsilon_Y, \quad \mathbb{E}[\epsilon_Y \mid XC] = 0,
\end{align}
where $a(C) = \tau^{\mathrm{BD}}_0(C)$ and $b(C) = \tau^{\mathrm{BD}}(C)$.

The BD-R-learner's loss function, equipped with an arbitrary nuisance $\widetilde{\eta} = \{\widetilde{e}_X, \widetilde{m}_Y\}$ is defined as:
\begin{align}
    \ell^{\mathrm{BDR}}_{\lambda}( (Y,X,C), \widetilde{\eta}, \tau) \triangleq \left(Y - \widetilde{m}_Y(C) - \{X - \widetilde{e}_X(C)\}\tau(C) \right)^2.
\end{align}
Population and empirical risk functions for samples $\mathcal{D}(Y,X,C) \triangleq \{(Y_i,X_i,C_i)\}_{i=1}^{|\mathcal{D}|}$ are give by:
\begin{align}
    L^{\mathrm{BDR}}_{\lambda}(\tau, \widetilde{\eta}) &\triangleq \mathbb{E}_{Y,X,C \sim P}\big[ \ell^{\mathrm{BDR}}_{\lambda}( (Y,X,C), \widetilde{\eta}, \tau)  \big] + \lambda \mathcal{J}(\tau) \label{eq:r-risk}\\
    \widehat{L}^{\mathrm{BDR}}_{\lambda, \mathcal{D}(Y,X,C)}(\tau, \widetilde{\eta}) &\triangleq \frac{1}{\vert \mathcal{D} \vert} \sum_{i: V_i \in \mathcal{D}} \ell^{\mathrm{BDR}}_{\lambda}( (Y_i,X_i,C_i), \widetilde{\eta}, \tau)  + \lambda \mathcal{J}(\tau),\label{eq:empirical-r-risk}
\end{align}
where $\lambda$ is a hyperparameter and $\mathcal{J}$ is a regularizer. BD-R-Learner is defined as follows:
\begin{definition}[\textbf{BD-R-Learner} \citep{nie2021quasi, foster2019orthogonal}]\hypertarget{def:BDR}{}\label{def:BDR}
    \normalfont
    The BD-R-Learner $\widehat{\tau}^{\mathrm{BDR}}$ is learned from the following procedure:
    \begin{enumerate}[leftmargin=*]
        \item Split the i.i.d. dataset $\mathcal{D}(Y,X,C) \triangleq \{(Y_i,X_i,C_i)\}_{i=1}^{2n}$ into $\mathcal{D}_1$ and $\mathcal{D}_2$.
        \item Learn $\widehat{\eta}\triangleq \{\widehat{e}_X, \widehat{m}_Y\}$ using $\mathcal{D}_1$.
        \item Find $\widehat{\tau}^{\mathrm{BDR}} \in \arg\min_{\tau \in \mathcal{T}} \widehat{L}^{BDR}_{\lambda, \mathcal{D}_2(YXC)}(\tau, \widehat{\eta})$.
    \end{enumerate}
\end{definition}
One may repeat steps (2–3) with the partitions swapped and average the two estimates. The debiasedness property of the BD-R-learner is captured by the following cited interface result; we use it as an assumed stage oracle inequality, not as a new empirical-process theorem.
\begin{proposition}[\textbf{Error Analysis of BD-R-Learner} \citep{nie2021quasi,foster2019orthogonal}]\hypertarget{prop:error-BDR}{}\label{prop:error-BDR}
    \normalfont
    Let $\|\cdot\|_p$ denote the $L_p(P)$ norm. Let $\widehat{\tau}^{\mathrm{BDR}}$ denote the BD-R-Learner estimate from Def.~\ref{def:BDR}. Let $a \lesssim b$ denote $a \leq b$ up to a constant factor. Suppose the BD-R target learner satisfies the cited high-probability excess-risk interface: with probability at least $1-\epsilon$,
    \begin{align}
        L^{\mathrm{BDR}}_{\lambda}(\widehat{\tau}^{\mathrm{BDR}}, \widehat{\eta})
        - \min_{\tau \in \mathcal{T}}L^{\mathrm{BDR}}_{\lambda}(\tau, \widehat{\eta})
        \leq \mathcal{R}^{\mathrm{BDR}}_{\mathcal{T}}(\epsilon, \widehat{\tau}^{\mathrm{BDR}}, \widehat{\eta}).
    \end{align}
    Then, on the same event,
    \begin{align}
        \| \widehat{\tau}^{\mathrm{BDR}} \!-\! \tau^{\mathrm{BDR}} \|^2_2 \lesssim \mathcal{R}^{\mathrm{BDR}}_{\mathcal{T}}(\epsilon, \widehat{\tau}^{\mathrm{BDR}}, \widehat{\eta})  + \| \widehat{e}_X - e_X \|^4_4 + \|\hat{m}_Y \!-\! m_Y \|^2_4 \|\widehat{e}_X \!-\! e_X \|^2_4.
    \end{align}
\end{proposition}
This result exhibits the property of debiasedness in the R-learner stage: the leading term is the target excess-risk bound, while the nuisance leakage enters through fourth-order and product terms. Thus an oracle-rate interpretation is valid only when these nuisance remainders are no larger than the stage excess-risk term under the overlap and moment conditions stated above.

\section{FD-DR-Learner}\label{sec:fd-dr-learner}
We first introduce the FD-DR-Learner, a learner for the conditional FD effect $\tau(C)$ based on a front-door pseudo-outcome. The key property is conditional: at the true nuisances, the pseudo-outcome has conditional mean $\tau(C)$; with estimated nuisances, the remaining bias is a conditional nuisance-bias term. The resulting guarantee is therefore an excess-risk bound plus a product-error remainder, rather than an unconditional oracle statement.

In addition to the nuisance functions $m(ZXC)$, $e(X \mid C)$ and $q(Z\mid XC)$ in Eq.~\eqref{eq:nuisance-outcome}, \eqref{eq:nuisance-propensity}, \eqref{eq:nuisance-mediator-propensity}, define the following weights for $\bar{x}\in\{0,1\}$:
\begin{align}
    \xi_{\bar{x}}(ZXC) \triangleq \frac{q(Z \mid \bar{x}C)}{q(Z \mid XC)},  \quad \pi_{\bar{x}}(XC) &\triangleq  \frac{\mathbb{I}(X=\bar{x})}{e(\bar{x} \mid C)}.
\end{align}
We also define the following induced functionals of $(m,e,q)$:
\begin{align}
    r_{me}(ZC) &\triangleq \sum_{x \in \{0,1\}}m(ZxC)e(x\mid C),\\
    \nu_{meq}(XC) &\triangleq \sum_{z \in \{0,1\}} r_{me}(zC) q(z \mid XC), \\
    s_{mq_{\bar{x}}}(XC) &\triangleq \sum_{z \in \{0,1\}} m(zXC)q(z \mid \bar{x}C).
\end{align}
Let $\eta_0\triangleq\{m,e,q\}$ denote the true front-door nuisance set. For a generic nuisance value $\widetilde\eta=\{\widetilde m,\widetilde e,\widetilde q\}$, the notation $\varphi_{\bar x}(V;\widetilde\eta)$ means that the weights and induced functionals in the display below are evaluated with $(m,e,q)$ replaced by $(\widetilde m,\widetilde e,\widetilde q)$.

\begin{definition}[\textbf{Front-door Pseudo-Outcome} (FDPO)]\label{def:fdpo}
    \normalfont
    For $\bar{x}\in\{0,1\}$ and nuisance set $\eta=\{m,e,q\}$, the \emph{front-door pseudo-outcome} at $\bar{x}$ is
    \begin{align}
        \varphi_{\bar{x}}(V;\eta)
        \triangleq
        \xi_{\bar{x}}(ZXC)\{Y \!-\! m(ZXC)\}
        + \pi_{\bar{x}}(XC)\{r_{me}(ZC) \!-\!\nu_{meq}(XC)\}
        + s_{mq_{\bar{x}}}(XC).
    \end{align}
\end{definition}
For any nuisance value $\eta$, write the difference FDPO as
\begin{align}
    D_{\eta}(V)\triangleq \varphi_{1}(V;\eta)-\varphi_{0}(V;\eta).
\end{align}

\begin{lemma}[\textbf{FDPO Conditional Mean}]\label{lemma:fdpo-analysis}
    \normalfont
    Under the assumptions in \S~\ref{sec:problem-setup}, for $\bar{x}\in\{0,1\}$,
    \begin{align}
        \mathbb{E}\!\big[\varphi_{\bar{x}}(V;\eta_0)\mid C\big]=\tau_{\bar{x}}(C).
    \end{align}
\end{lemma}

\begin{lemma}[\textbf{FDPO Conditional Bias Target}]\label{lemma:fdpo-bias}
    \normalfont
    Let $\widehat{\eta}$ be any nuisance estimate satisfying the denominator and moment requirements in \S~\ref{sec:problem-setup}. Define
    \begin{align}
        B_{\bar{x}}(C)
        \triangleq
        \mathbb{E}\!\big[\varphi_{\bar{x}}(V;\widehat{\eta})\mid C\big]-\tau_{\bar{x}}(C).
    \end{align}
    Equivalently, for $j\in\{0,1\}$,
    \begin{align}
        B_j(C)
        \triangleq
        \mathbb{E}\!\big[\varphi_j(V;\widehat{\eta})\mid C\big]-\tau_j(C).
    \end{align}
    With the difference FDPO $D_{\widehat{\eta}}$ above, we have
    \begin{align}
        \mathbb{E}\!\big[D_{\widehat{\eta}}(V)\mid C\big]-\tau(C)=B_{1}(C)-B_{0}(C).
    \end{align}
\end{lemma}

We learn $\tau(C)$ by regressing the FDPO difference on $C$. Specifically, define the population and empirical risk functions
\begin{align}\label{eq:risk-function-dr}
    \mathcal{L}^{\mathrm{DR}}_{\lambda}(\tilde{\tau},\tilde{\eta})
    &\triangleq
    \mathbb{E}\!\big[\big\{D_{\tilde{\eta}}(V)-\tilde{\tau}(C)\big\}^{2}\big]
    + \lambda\,\mathcal{J}(\tilde{\tau}),\\[2mm]
\label{eq:empirical-risk-function-dr}
    \widehat{\mathcal{L}}^{\mathrm{DR}}_{\lambda,\mathcal{D}}(\tilde{\tau},\tilde{\eta})
    &\triangleq
    \frac{1}{|\mathcal{D}|}\sum_{i:\,V_i\in\mathcal{D}}
    \big(\big\{D_{\tilde{\eta}}(V_i)-\tilde{\tau}(C_i)\big\}^{2}
    + \lambda\,\mathcal{J}(\tilde{\tau})\big).
\end{align}

\begin{definition}[\textbf{FD-DR-Learner}]\hypertarget{def:fd-dr-learner}{}\label{def:fd-dr-learner}
    \normalfont
    Let $\mathcal{D}_1,\mathcal{D}_2$ be two disjoint splits of i.i.d. observations $V_i=(C_i,X_i,Z_i,Y_i)$. Let $\mathcal{T}$ be the model class for $\tau$, and let $\lambda_n$ be a regularization level.
    \begin{enumerate}[leftmargin=*]
        \item Fit $\hat{\eta}\triangleq\{\hat{m},\hat{q}, \hat{e}\}$ on $\mathcal{D}_1$.
        \item Compute $\displaystyle \hat{\tau}_{\mathrm{DR}}\in\arg\min_{\tilde{\tau}\in\mathcal{T}}
\widehat{\mathcal{L}}^{\mathrm{DR}}_{\lambda_n,\mathcal{D}_2}(\tilde{\tau},\hat{\eta})$ using $\mathcal{D}_2$.
        \item Optionally swap the roles of $\mathcal{D}_1$ and $\mathcal{D}_2$ and average the two estimates.
    \end{enumerate}
\end{definition}

\begin{theorem}[\textbf{Error Analysis of FD-DR-Learner}]\label{thm:error-fd-dr}
    \normalfont
    Let
    \begin{align}
        \mu_{\widehat{\eta}}(C)
        &\triangleq\mathbb{E}\!\big[D_{\widehat{\eta}}(V)\mid C\big],\\
        B_j(C)
        &\triangleq
        \mathbb{E}\!\big[\varphi_j(V;\widehat{\eta})\mid C\big]-\tau_j(C),
        \qquad j\in\{0,1\},\\
        R_{\mathrm{DR}}
        &\triangleq
        \|\widehat{\tau}_{\mathrm{DR}}-\mu_{\widehat{\eta}}\|_2^2.
    \end{align}
    Here $B_j$ is the conditional FDPO bias at intervention level $j$, and $R_{\mathrm{DR}}$ is the final-stage target-regression error relative to the conditional mean of the estimated FDPO difference. Then, on the event under analysis,
    \begin{align}
        \| \widehat{\tau}_{\mathrm{DR}}-\tau \|_2^2
        \le
        2R_{\mathrm{DR}}+2\|B_1-B_0\|_2^2.
    \end{align}
\end{theorem}

\begin{corollary}[\textbf{Conditional Quasi-Oracle Statement for FD-DR}]\label{cor:fd-dr-quasi-oracle}
    \normalfont
    Let $r_{\mathrm{DR},n}$ be a rate such that $R_{\mathrm{DR}}=O_p(r_{\mathrm{DR},n})$. If $\|B_1-B_0\|_2^2=o_p(r_{\mathrm{DR},n})$, then the FD-DR learner has the same first-order rate as the final target regression oracle. If the FDPO conditional bias is second order in nuisance errors and all nuisance factors entering those products are $O_p(n^{-1/4})$, then $\|B_1-B_0\|_2^2=O_p(n^{-1})$. Consequently, a quasi-oracle interpretation is justified only when this second-order remainder is no larger than the target oracle rate under the assumptions in \S~\ref{sec:problem-setup}.
\end{corollary}

In words, FD-DR is debiased because the nuisance contribution appears through the conditional bias term $B_1-B_0$. We use ``double robustness'' in this paper in this product-error sense: the bound improves when the relevant nuisance products are small. It should not be read as a claim that arbitrary misspecification of one nuisance block is harmless without the stated overlap, moment, and conditional-bias assumptions.

\section{FD-R-Learner}\label{sec:fd-r-learner}
In this section, we introduce the FD-R-Learner, a learner for the conditional FD effect $\tau(C)$ based on a pathway decomposition. The guarantee is stated through explicit stage-error bounds: one stage estimates the effect of $X$ on $Z$, another estimates the effect of $Z$ on $Y$ conditional on $(X,C)$, and a final regression estimates the conditional average of the second pathway component.
Let $\mathcal{B},\mathcal{Q},\Gamma$ denote the function classes used to learn $b$, $g$, and $\gamma$, respectively.

\begin{proposition}[\textbf{Partial Linear Representation for FD Stages}]\label{prop:ple-FD}
    \normalfont
    Under the conditional front-door assumptions in \S~\ref{sec:problem-setup}, define
    \begin{align*}
        e_X(C)&\triangleq \mathbb{E}[X\mid C],\\
        a(C)&\triangleq \mathbb{E}[Z\mid X=0,C],\\
        b(C)&\triangleq \mathbb{E}[Z\mid X=1,C]-\mathbb{E}[Z\mid X=0,C],\\
        f(XC)&\triangleq \mathbb{E}[Y\mid Z=0,X,C],\\
        g(XC)&\triangleq \mathbb{E}[Y\mid Z=1,X,C]-\mathbb{E}[Y\mid Z=0,X,C].
    \end{align*}
    Then the observed conditional means admit the residual representations
    \begin{align}
        X &= e_X(C) + \epsilon_X, \quad \mathbb{E}[\epsilon_X \mid C] = 0, \\
        Z &= a(C) + X b(C) + \epsilon_Z, \quad \mathbb{E}[\epsilon_Z \mid XC] = 0, \label{eq:robinson-z} \\
        Y &= f(XC) + Z g(XC) + \epsilon_Y, \quad \mathbb{E}[\epsilon_Y \mid ZXC] = 0.
    \end{align}
\end{proposition}

Since $C$ satisfies the back-door criterion relative to $(X,Z)$, the pathway component $b(C)$ can be learned using the \hyperlink{def:BDR}{BD-R-Learner}:
\begin{align}\label{eq:r-learner-b}
    \widehat{b} &\in \arg\min_{b \in \mathcal{B}}\widehat{L}^{\mathrm{BDR}}_{\lambda_b, \mathcal{D}_{2}(Z,X,C)}(b, \widehat{\eta}_b \triangleq \{\widehat{e}_X, \widehat{m}_Z\}),
\end{align}
where $\widehat{m}_Z$ estimates $m_Z(C)\triangleq \mathbb{E}[Z\mid C]$. The corresponding stage guarantee must be checked or assumed explicitly; in the notation below it is expressed through the oracle stage error $R_b$ and nuisance remainder $N_b$.

Similarly, since $(X,C)$ satisfies the back-door criterion relative to $(Z,Y)$, the pathway component $g(XC)$ can be learned through an R-learner stage:
\begin{align}\label{eq:r-learner-g}
    \widehat{g} &\in \arg\min_{g \in \mathcal{Q}}\widehat{L}^{\mathrm{BDR}}_{\lambda_g, \mathcal{D}_{2}(Y,Z,XC)}(g, \widehat{\eta}_g \triangleq \{\widehat{e}_Z, \widehat{m}_Y\}),
\end{align}
where $\widehat{e}_Z$ estimates $e_Z(XC)\triangleq \mathbb{E}[Z\mid X,C]$ and $\widehat{m}_Y$ estimates $m_Y(XC)\triangleq \mathbb{E}[Y\mid X,C]$. Its BD-R contribution is carried forward through the oracle error $R_g$ and nuisance remainder $N_g$ in the final theorem; the final pseudo-$g$ regression has its own terms $R_\gamma$ and $N_\gamma$.

\begin{theorem}[\textbf{Heterogeneous Treatment Effect via Pathway Decomposition}]\label{thm:fd-hte-ple}
    \normalfont
    Let $b$ and $g$ denote the functionals in Prop.~\ref{prop:ple-FD}. Define
    \begin{align}
        \gamma_g(C) \triangleq \mathbb{E}[g(XC) \mid C]
        = e_X(C)g(1C)+\{1-e_X(C)\}g(0C).
    \end{align}
    Then,
    \begin{align}
        \tau(C) = b(C) \gamma_g(C).
    \end{align}
\end{theorem}

A remaining challenge is to estimate $\gamma_g(C)$ without inheriting a first-order error from $\widehat e_X$. A plug-in estimator,
\begin{align}
    \widehat{\gamma}^{\mathrm{plug}}(C) \triangleq  \widehat{e}_X(C)\widehat{g}(1C) + \{1-\widehat{e}_X(C)\}\widehat{g}(0C),
\end{align}
contains the term $\{g(1C)-g(0C)\}\{\widehat e_X(C)-e_X(C)\}$ in its first-order expansion. We instead use the following pseudo-$g$ target.

\begin{definition}[\textbf{Pseudo-$g$}]\hypertarget{def:pseudo-g}{}\label{def:pseudo-g}
    \normalfont
    For $\widetilde{\eta}_{z} \triangleq \{\tilde{e}_X(C), \tilde{g}(XC) \}$, define
    \begin{align}
        \zeta_{\widetilde{\eta}_z}(XC) \triangleq
        \{1 \!-\! \widetilde{e}_X(C)\}\widetilde{g}(0C)
        + \widetilde{e}_X(C)\widetilde{g}(1C)
        + \{X \!-\! \widetilde{e}_X(C)\}\{ \widetilde{g}(1C) \!-\! \widetilde{g}(0C) \}.
    \end{align}
\end{definition}

\begin{lemma}[\textbf{Property of pseudo-$g$}]\label{lemma:property-pseudo-g}
    \normalfont
    For any estimated $\widehat{e}_X$ and $\widehat{g}$ satisfying the moment conditions in \S~\ref{sec:problem-setup}, set $\widehat{\eta}_{z}\triangleq\{\widehat e_X,\widehat g\}$. Then
    \begin{align}
        \mathbb{E}[\zeta_{\widehat{\eta}_{z}}(XC)\mid C]-\gamma_g(C)
        =
        e_X(C) \{\widehat{g}(1C)-g(1C)\}
        +\{1-e_X(C)\}\{\widehat{g}(0C)-g(0C)\}.
    \end{align}
    In particular, the conditional bias contains no first-order term in $\widehat e_X-e_X$.
\end{lemma}

We now formally define the FD-R-Learner.
\begin{definition}[\textbf{FD-R-Learner}]\hypertarget{def:fd-r-learner}{}\label{def:fd-r-learner}
    \normalfont
    Let $(\mathcal{D}_1, \mathcal{D}_2, \mathcal{D}_3)$ be disjoint splits of i.i.d. observations $V_i = (C_i,X_i,Z_i,Y_i)$.
    \begin{enumerate}[leftmargin=*]
        \item Fit $\widehat{\eta}_{b} \triangleq \{\widehat{e}_X, \widehat{m}_Z\}$ and $\widehat{\eta}_g \triangleq \{\widehat{e}_Z, \widehat{m}_Y\}$ using $\mathcal{D}_1$.
        \item Learn $\widehat{b}$ using $\mathcal{D}_{2}$ with $\widehat{\eta}_{b}$ from Eq.~\eqref{eq:r-learner-b}.
        \item Learn $\widehat{g}$ using $\mathcal{D}_{2}$ with $\widehat{\eta}_{g}$ from Eq.~\eqref{eq:r-learner-g}.
        \item Set $\widehat{\eta}_{z}\triangleq\{\widehat e_X,\widehat g\}$ and evaluate $\zeta_{\widehat{\eta}_{z}}(XC)$ on $\mathcal{D}_3$.
        \item Find
        \begin{align}
            \widehat{\gamma} \in \arg\min_{\gamma \in \Gamma} \frac{1}{|\mathcal{D}_3|} \sum_{i:V_i \in \mathcal{D}_3} \{\zeta_{\widehat{\eta}_{z}}(X_iC_i) - \gamma(C_i)\}^2 + \lambda_{\gamma,n}\mathcal{J}(\gamma).
        \end{align}
        \item Return $\widehat{\tau}_{\mathrm{R}}(C) \triangleq \widehat{b}(C) \widehat{\gamma}(C)$.
    \end{enumerate}
\end{definition}
One may repeat the procedure with swapped splits and average the estimates.

The next result is a stage-error decomposition. It is not a fast-rate theorem by itself: it becomes a rate theorem only after one supplies separate rate bounds for the stage oracle errors and nuisance remainders.
\begin{theorem}[\textbf{Stage-Error Decomposition for FD-R-Learner}]\label{thm:error-fd-r}
    \normalfont
    Let $a\lesssim b$ denote $a\le Kb$ for a universal constant $K$ independent of the sample size.
    For any stage function $h(X,C)$, write
    \begin{align}
        \gamma_h(C)\triangleq \mathbb{E}[h(X,C)\mid C].
    \end{align}
    Let $\bar b$ and $\bar g$ denote the oracle versions of the $b$- and $g$-stage learners obtained by running the same learning rules with the true nuisance functions in place of $\widehat\eta_b$ and $\widehat\eta_g$. Let $\eta_z^\star\triangleq\{e_X,\bar g\}$ and write $\zeta_{(e_X,\bar g)}\equiv\zeta_{\eta_z^\star}$. Let $\bar\gamma$ denote the oracle final-stage learner obtained by running the same $\gamma$-regression with the oracle pseudo-outcome $\zeta_{(e_X,\bar g)}(X,C)$, so its regression target is $\gamma_{\bar g}$.
    Define the oracle stage errors
    \begin{align}
        R_b&\triangleq \|\bar b-b\|_{2,C}^2,&
        R_g&\triangleq \|\bar g-g\|_{2,XC}^2,&
        R_\gamma&\triangleq \|\bar\gamma-\gamma_{\bar g}\|_{2,C}^2,
    \end{align}
    and the nuisance remainders
    \begin{align}
        N_b&\triangleq \|\widehat b-\bar b\|_{2,C}^2,\\
        N_g&\triangleq \|\widehat g-\bar g\|_{2,XC}^2,\\
        N_\gamma&\triangleq \|\widehat\gamma-\bar\gamma\|_{2,C}^2.
    \end{align}
    Here $\|\cdot\|_{2,C}$ and $\|\cdot\|_{2,XC}$ denote $L_2$ norms under the marginal laws of $C$ and $(X,C)$, respectively, conditional on the training splits.
    The pairs $R_b+N_b$ and $R_g+N_g$ are exactly the BD-R oracle-plus-nuisance decompositions for the $X\!\to\!Z$ and $Z\!\to\!Y$ back-door subproblems. The term $R_\gamma+N_\gamma$ belongs to the final pseudo-$g$ regression interface.

    Suppose also that $\|\widehat b\|_\infty\le M_b$ and $\|\gamma_g\|_\infty\le M_\gamma$. Then
    \begin{align}
        \|\widehat\tau_{\mathrm{R}}-\tau\|_{2,C}^2
        \le
        2M_b^2\|\widehat\gamma-\gamma_g\|_{2,C}^2
        +2M_\gamma^2\|\widehat b-b\|_{2,C}^2.
    \end{align}
    Moreover, the definitions above imply the stage-interface bounds
    \begin{align}
        \|\widehat b-b\|_{2,C}^2 &\lesssim R_b+N_b,\\
        \|\widehat g-g\|_{2,XC}^2 &\lesssim R_g+N_g,\\
        \|\widehat\gamma-\gamma_g\|_{2,C}^2 &\lesssim R_\gamma+R_g+N_\gamma,
    \end{align}
    and therefore
    \begin{align}
        \|\widehat\tau_{\mathrm{R}}-\tau\|_{2,C}^2
        \lesssim R_b+R_g+R_\gamma+N_b+N_g+N_\gamma.
    \end{align}
    Thus the theorem decomposes the error into named stage terms; it does not assert any particular convergence speed for those terms.
\end{theorem}

\begin{corollary}[\textbf{Conditional Quasi-Oracle Statement for FD-R}]\label{cor:fd-r-quasi-oracle}
    \normalfont
    If $R_b+R_g+R_\gamma=O_p(r_n)$ and $N_b+N_g+N_\gamma=o_p(r_n)$, then FD-R has the same first-order rate as the oracle stage terms. If the BD-R nuisance leakages $N_b$ and $N_g$ are fourth-order or second-order product terms, all nuisance factors entering those products are $O_p(n^{-1/4})$, and the final pseudo-$g$ regression leakage $N_\gamma$ is also lower order than $r_n$, then FD-R has a quasi-oracle or fast-rate interpretation. This interpretation requires explicit stage-rate guarantees and the condition that the BD-R and final-regression nuisance remainders do not dominate the stage oracle rates.
\end{corollary}

\subsection{Comparison between FD-DR-Learner and FD-R-Learner}\label{subsec:comparison}

\textbf{FD-DR-Learner.} FD-DR has product-error debiasedness (Thm.~\ref{thm:error-fd-dr}): its error is controlled by the target regression term plus the conditional FDPO bias $\|B_1-B_0\|_2^2$. This makes FD-DR attractive when the conditional bias products are small and bounded overlap is credible. A caveat is variance under near-violations of overlap: its construction uses inverse weights and density ratios (e.g., $\pi_{\bar x}(XC)$ and $\xi_{\bar x}(ZXC)$), which can inflate finite-sample variance when $e$ or $q$ approach $0$ or $1$.

\textbf{FD-R-Learner.} FD-R avoids the density ratios required by FD-DR, making it more variance-friendly in finite samples when overlap is weak. This does not remove the formal bounded-overlap requirement in the theorem, but it can make FD-R a practical choice when inverse-weight diagnostics are unstable. FD-R also decomposes estimation into two BD-R-style pathway components, $b(C)$ for $X\!\to\!Z$ and $g(XC)$ for $Z\!\to\!Y$, plus the final $\gamma_g(C)$ regression. These intermediates are useful for diagnostics, at the cost of additional nuisance and stage fits.

\textbf{Practitioner Guidance.} Use FD-DR when FDPO products are controlled and overlap is stable; use FD-R when avoiding density ratios or inspecting pathways matters, subject to the stated assumptions.

\begin{figure*}[t]
    \centering
    \adjincludegraphics[scale=0.20,trim={{.00\width} {.16\width} {.0\width} {.15\width}},clip]{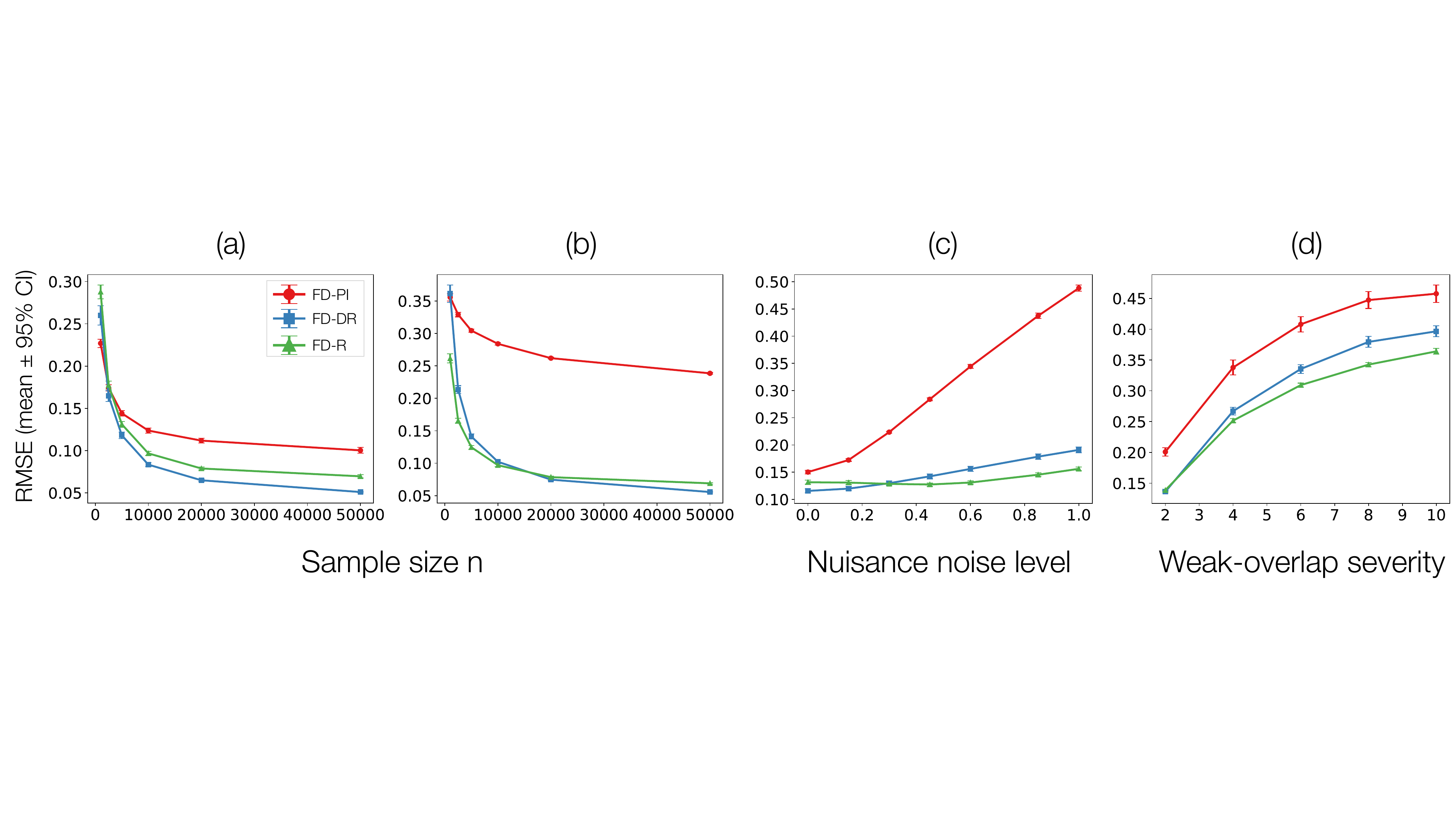}
    \caption{Synthetic study: RMSE (mean ± 95\% CI) of FD-PI (plug-in), FD-DR, and FD-R estimators: \textbf{(a-b)} varying sample size $n$, where (a) no structural nuisance noise is added, and (b) simulation-level nuisance perturbations are set at the $n^{-1/4}$ scale; \textbf{(c)} varying the level $\rho$ of nuisance noises $\rho\epsilon, \epsilon\sim\mathcal{N}(n^{-1/4}, n^{-1/4})$; and \textbf{(d)} varying weak-overlap severity.}
    \label{fig:synthetic}
\end{figure*}

\section{Experiments}\label{sec:experiments}
In this section, we demonstrate the debiasedness of the proposed estimators. In all experiments, nuisance functions are learned with XGBoost \citep{chen2016xgboost}. We compare the proposed FD-DR and FD-R-learners with a plug-in (PI) estimator $\widehat{\tau}_{\mathrm{PI}}$ of the \hyperlink{eq:target-estimator}{target estimand}:
\begin{align}
    \widehat{\tau}_{\mathrm{PI}}(C) \triangleq \sum_{zx}\{\widehat{q}(z \mid 1C) - \widehat{q}(z \mid 0C)\}\widehat{e}(X=x \mid C)\widehat{m}(zxC).
\end{align}
Details of simulations are described in Sec.~\ref{sec:sim-details}. The implementation is available at \url{https://github.com/yonghanjung/FD-CATE}.

\subsection{Synthetic Data Analysis}\label{subsec:synthetic-data-analysis}
We assess the proposed estimators on synthetic scenarios where the true heterogeneous FD effect $\tau(C)$ is known. Fig.~\ref{fig:synthetic} reports the root mean squared error (RMSE; mean $\pm$ 95\% CI) of the plug-in baseline (FD-PI, $\widehat{\tau}_{\mathrm{PI}}$), the FD-DR learner, and the FD-R learner across four regimes.

Panels~(a-b) vary the sample size $n$. In these simulations, FD-DR and FD-R have lower RMSE than FD-PI as $n$ grows. When no additional structural noise is introduced (panel~(a)), all estimators converge as expected, but the proposed FD-DR and FD-R learners achieve substantially lower error. When the simulation imposes nuisance perturbations at the $n^{-1/4}$ scale (panel~(b)), FD-DR and FD-R still exhibit reliable empirical convergence, consistent with the product-error and stage-error robustness predicted by the theory under its stated assumptions, whereas the plug-in estimator converges much more slowly.

To further probe robustness under noisy nuisances, we inject controlled simulation noise of the form $\rho \epsilon$ into the nuisance functionals, where $\epsilon \sim \mathcal{N}(n^{-1/4}, n^{-1/4})$ and $\rho \in \{0.0,0.2,0.4,0.6,0.8,1.0\}$. Panel~(c) shows that the RMSE of $\widehat{\tau}_{\mathrm{PI}}$ deteriorates rapidly as $\rho$ increases, while FD-DR and FD-R maintain stable performance. FD-R achieves lower RMSE across higher noise levels, reflecting its reduced reliance on inverse-weighted nuisance ratios in this simulation.

Finally, panel~(d) examines weak-overlap scenarios by pushing $e(X=1 \mid C)$ and $q(Z=1 \mid X,C)$ toward zero. In this regime, the plug-in estimator exhibits severe degradation, and FD-DR suffers noticeable variance inflation due to its use of inverse weights. FD-R, by contrast, remains stable and outperforms both FD-PI and FD-DR in this simulation. These patterns mirror our analytical comparison in \S~\ref{subsec:comparison}, and are consistent with FD-DR performing best when nuisances are accurately estimated and overlap is sufficient, while FD-R is particularly advantageous under weak overlap or noisy nuisance models.

\subsection{Case Study: State Seat-Belt Laws and Fatalities (FARS)}\label{subsec:real-data}

\begin{figure*}[t]
    \centering
    \adjincludegraphics[scale=0.20,trim={{.00\width} {.16\width} {.0\width} {.10\width}},clip]{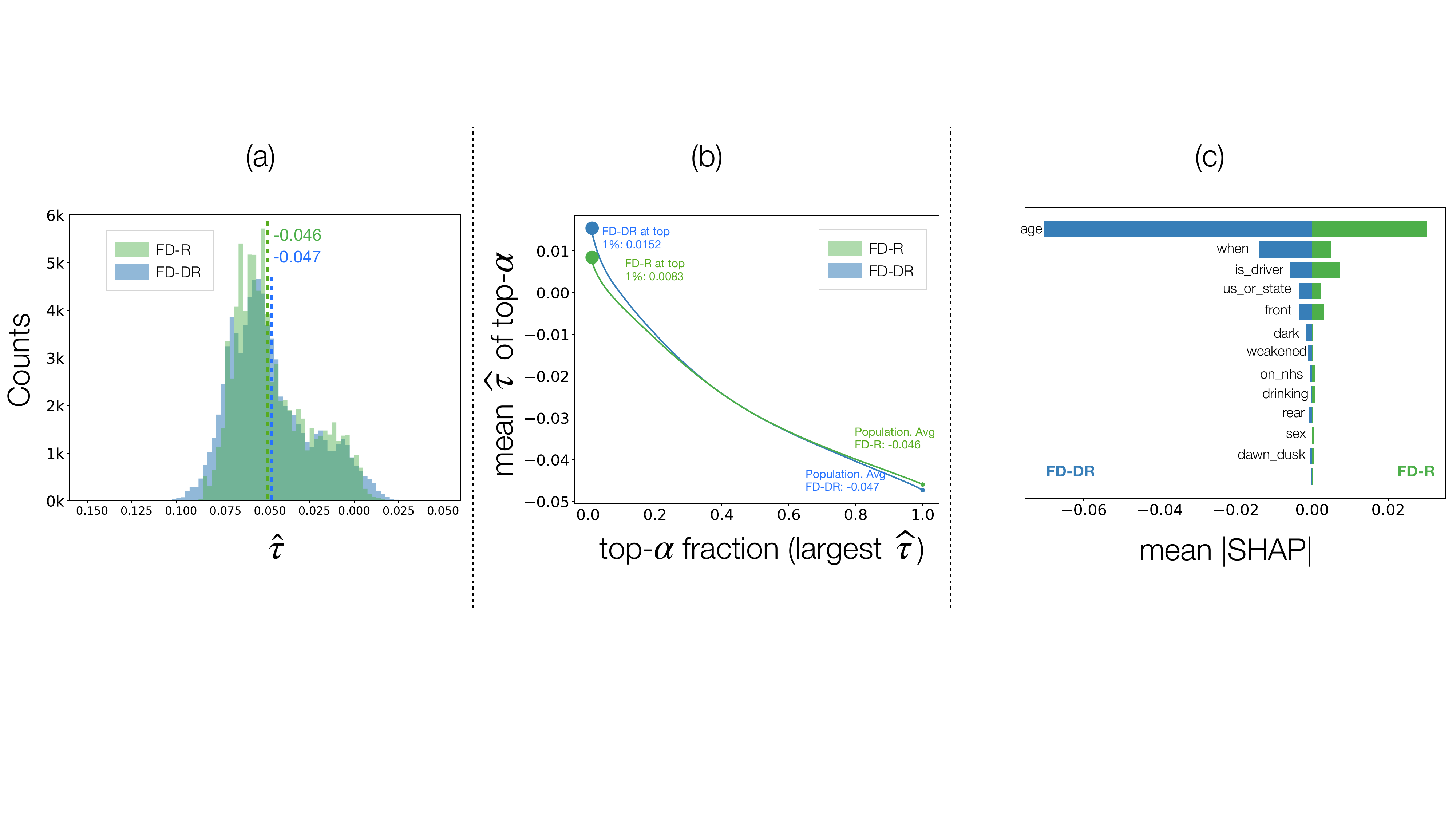}
    \caption{Seat-Belt Laws and Fatalities: \textbf{(a)} histograms for $\widehat{\tau}_{\mathrm{DR}}(C_i)$ and $\widehat{\tau}_{\mathrm{R}}(C_i)$ \textbf{(b)} concentration curve showing the mean values of $\widehat\tau_{\mathrm{DR}}$ and $\widehat\tau_{\mathrm{R}}$ within the top-$\alpha$ fraction (largest $\widehat\tau$), where both learners exhibit a downward trend; \textbf{(c)} SHAP feature importance highlighting age, time of date, and driver status as the most influential features for both learners.}
    \label{fig:real}
\end{figure*}

We use a state–year panel constructed from National Highway Traffic Safety Administration (NHTSA) sources \citep{nhtsa_fars2000} using the Fatality Analysis Reporting System (FARS). In our data, the treatment is whether a state–year has a \emph{primary seat-belt law} (enforcing seat-belt use) ($X$), the mediator is the \emph{belt-use} ($Z$) from NHTSA surveys, the outcome is the \emph{occupant fatality} $Y$, and $C$ collects covariates affecting $(X,Z,Y)$. The front-door structure is plausible here because the effect from $X$ to $Y$ operates through increased belt use $Z$; rich set of covariates $C$ helps mitigate confounding bias between $(X, Z)$ and between $(Z, Y)$.  Also, positivity holds in this dataset, because belt use is neither zero nor universal across law regimes; i.e., some occupants do not fasten seat belts even under a primary law.

Fig.~\ref{fig:real} summarizes the analysis. Panel~(a) shows the distributions of $\widehat{\tau}_{\mathrm{DR}}(C_i)$ and $\widehat{\tau}_{\mathrm{R}}(C_i)$, indicating lower fatality rates under primary-law regimes. Panel~(b) reports the \textit{concentration curve} (mean $\widehat{\tau}$ within the top-$\alpha$ fraction ranked by $\widehat{\tau}$). Only a small portion of units exhibit increases fatality under the primary law regime ($X{=}1$), whereas over $95\%$ of units show decreases, consistent with a preventive effect of primary-law adoption. Panel~(c) presents SHAP feature importance, highlighting age, time of day, and driver status as dominant factors explaining heterogeneity.

Together, our results indicate that primary seat-belt laws reduce occupant fatality rates for the majority of units, illustrating the practical utility of our approach for estimating causal effects from real-world datasets that fit the FD setting.

\section{Conclusion and Discussion}

\textbf{Summary}. We developed two heterogeneous FD treatment effects estimators: \hyperlink{def:fd-dr-learner}{FD-DR-Learner} and \hyperlink{def:fd-r-learner}{FD-R-Learner}. Under the stated sample-splitting, bounded-overlap, moment, and stage-learning assumptions, FD-DR satisfies a product-error bound and FD-R satisfies a stage-error decomposition; these guarantees yield conditional \emph{quasi-oracle} corollaries when the second-order nuisance remainders are no larger than the relevant oracle terms (Thm.~\ref{thm:error-fd-dr},~\ref{thm:error-fd-r}). A comparison of the two estimators for practitioners is given in \S~\ref{subsec:comparison}. In synthetic stress tests (varying $n$, slow nuisances, weak overlap) they dominate a plug-in baseline, and in our FARS seat-belt case study they deliver reliable personalized FD estimates.

\textbf{Limitations \& future work.} (i) \emph{Positivity.} Our guarantees assume bounded overlap for $e(X\!\mid\!C)$ and $q(Z\!\mid\!X,C)$; near-violations inflate variance (especially for FD-DR, whose pseudo-outcome uses inverse weights and density ratios). We recommend overlap diagnostics, ratio stabilization, and overlap-aware uncertainty, and plan adaptive routing toward FD-R under weak overlap as a finite-sample strategy rather than as a relaxation of the formal assumption. (ii) \emph{Binary mediator.} Our theory uses a binary $Z$, whereas many practical settings feature continuous or multidimensional mediators. Extending FD-DR and FD-R learners to handle such settings is a promising direction.

\endgroup

\bibliographystyle{apalike}
\bibliography{reference.bib}


\clearpage
\appendix 

\counterwithin*{lemma}{section}	
\counterwithin*{definition}{section}	
\counterwithin*{theorem}{section}	
\counterwithin*{proposition}{section}	
\counterwithin*{equation}{section}	
\counterwithin*{corollary}{section}	
\counterwithin*{remark}{section}	
\counterwithin*{example}{section}
\counterwithin*{assumption}{section}

\onecolumn
\vtop{
    \centering
    \vspace{0cm}
    \hspace{0cm}
    \begin{tikzpicture}
        \centering 
        \node (box){
          \begin{minipage}[t!]{0.9\textwidth}
            \centering
            {\Large Supplement of \Paste{title} }
          \end{minipage}
        };
    \end{tikzpicture}
    \vspace{0.5cm}
}
\begingroup
\allowdisplaybreaks
\sloppy
\section{Preliminaries: R-Learner and DR-Learner Analysis using Orthogonal Statistical Learning}
This appendix uses generic back-door notation $(T,X,Y)$; it is separate from the front-door notation $(X,Z,C,Y)$ used in the main text.
The theorem-like statements here are generic OSL support templates, whereas the main front-door claims are Thms.~\ref{thm:error-fd-dr} and~\ref{thm:error-fd-r}.

We study a population risk $L(\tau,\eta)$, where the \emph{target} $\tau\in\mathcal T$ and the \emph{nuisance} $\eta\in\mathcal H$ live in normed spaces $(\mathcal T,\|\cdot\|_{\mathcal T})$ and $(\mathcal H,\|\cdot\|_{\mathcal H})$, respectively. Throughout, $\eta_0$ denotes the true nuisance. We define the (possibly non-unique) \emph{oracle minimizer}
\begin{align}
    \tau_0 \in \arg\min_{\tau\in\mathcal T} L(\tau,\eta_0),
\end{align}
which we assume is nonempty.

\paragraph{Directional derivatives.}
For a functional $F$ and direction $h$, the (Gâteaux) derivative with respect to a variable $x$ at $x_0$ is
\begin{align}
    \nabla_{x}F(x_0)[h] \;\triangleq\; \lim_{t\to 0}\frac{F(x_0 + th) - F(x_0)}{t},
\end{align}
and second derivatives $\nabla^2_x F(x_0)[h_1,h_2]$ are defined analogously; mixed derivatives such as $\nabla_\eta\nabla_\tau L$ will be used for orthogonality.

\paragraph{Sample splitting and plug-in.}
For any nuisance value $\eta$, define
\[
\tau^*_{\eta}\ \in\ \arg\min_{\tau\in\mathcal T} L(\tau,\eta).
\]
We assume a two-way split into independent folds of approximately equal size: one to learn $\hat\eta$ (using data $\mathcal D_\eta$), and one to learn $\hat\tau$ by minimizing $L(\tau,\hat\eta)$ over $\tau\in\mathcal T$, so that
\[
\tau_0=\tau^*_{\eta_0}.
\]
We assume the displayed minimizers exist on the high-probability event analyzed; otherwise the same statements can be read with the corresponding infimum values.
This separation prevents overfitting-induced bias when we later linearize around $(\tau_0,\eta_0)$.

\paragraph{Target-class statistical term.}
Let $R_{\mathcal T}(\tau;\eta,\epsilon)\ge 0$ be a data-dependent rate function such that, with probability at least $1-\epsilon$,
\begin{align}
    L(\tau,\eta) - L(\tau^{\ast}_{\eta}, \eta) \;\le\; R_{\mathcal T}(\tau;\eta,\epsilon).
\end{align}
You may instantiate $R_{\mathcal T}$ via localized complexity (e.g., critical radius) or algorithm-specific bounds; we keep it abstract to highlight how nuisance error propagates into target error.

\paragraph{Goal and norms.}
Our goal is to upper bound the \emph{target error} $\|\tau-\tau_0\|_{\mathcal T}^2$. When we write $\|\cdot\|_p$ we mean the $L_p(P)$ norm with respect to the underlying distribution.

\subsection{Examples (R- and DR-learners)}
We use standard notation: $T\in\{0,1\}$ (treatment), $X$ (covariates), $Y$ (outcome). The estimand is the CATE
\begin{align}
    \tau_0(X) \triangleq \mathbb{E}[Y(1)-Y(0)\mid X],
\end{align}
under the usual \emph{positivity} ($c\le \pi_0(X)\le 1-c$ a.s.) and i.i.d.\ sampling. We assume
\begin{align}
    Y(t) \perp\!\!\!\perp T \mid X \implies \mathbb{E}[Y(t) \mid X] = \mathbb{E}[Y \mid t,X], \; \forall t \in \{0,1\}.
\end{align}

\subsubsection{R-Learner}
The Robinson decomposition posits
\begin{align}
    Y &= f_0(X) + T\,\tau_0(X) + \epsilon_Y,\quad \mathbb{E}[\epsilon_Y\mid T,X]=0,\\
    T &= \pi_0(X) + \epsilon_X,\quad \mathbb{E}[\epsilon_X\mid X]=0,
\end{align}
and with $m_0(X)\triangleq \mathbb{E}[Y\mid X]$ we have $m_0(X)=f_0(X)+\pi_0(X)\tau_0(X)$. Hence
\begin{align}
    Y - m_0(X) \;=\; \big(T-\pi_0(X)\big)\,\tau_0(X) + \epsilon_Y.
\end{align}
Thus, viewing $\tau_0$ as an OLS-type coefficient in a residualized regression, we define
\begin{align}
    L_{\mathrm{R}}(\tau,\eta_0\!\triangleq\!\{m_0,\pi_0\})
    \;\triangleq\;
    \mathbb{E}\!\left[\big\{Y-m_0(X) - \big(T-\pi_0(X)\big)\tau(X)\big\}^2\right],
\end{align}
so that $\tau_0\in\arg\min_\tau L_{\mathrm R}(\tau,\eta_0)$.

\subsubsection{DR-Learner}
We define following nuisances:
\begin{align}
    \mu_0(T,X) &\triangleq \mathbb{E}[Y \mid T,X], \qquad \omega_0(T,X) \triangleq \frac{2T-1}{P(T \mid X)}.
\end{align}
Define the pseudo-outcome
\begin{align}
    \varphi(V;\eta_0\!\triangleq\!\{\mu_0,\omega_0\})
    &\triangleq \omega_0(T,X)\{Y-\mu_0(T,X)\} + \mu_0(1,X)-\mu_0(0,X),
\end{align}
and the squared-loss objective
\begin{align}
    L_{\mathrm{DR}}(\tau,\eta)\;\triangleq\; \mathbb{E}\!\left[\big\{\varphi(V;\eta)-\tau(X)\big\}^2\right].
\end{align}
This loss is centered at the CATE in virtue of $\mathbb E[\varphi(V;\eta_0)\mid X]=\tau_0(X)$.

\subsection{Assumptions}
We now state structural conditions that yield fast rates. The exposition follows the orthogonal-statistical-learning (OSL) template: first-order optimality at truth, curvature in $\tau$, and orthogonality to damp the impact of nuisance error.

\begin{assumption}[\textbf{First-order optimality in $\tau$}]\label{assumption:first-order-optimality}
    \normalfont
    Moving away from $\tau_0$ cannot reduce the population risk at the true nuisance:
    \begin{align}\label{eq:first-order-optimality}
        \nabla_{\tau} L(\tau_0,\eta_0)[h_{\tau}] \;\ge\; 0
        \quad \text{for all feasible directions $h_\tau$ from $\tau_0$.}
    \end{align}
\end{assumption}

\begin{assumption}[\textbf{Strong convexity (quadratic growth) in $\tau$}]\label{assumption:strong-convexity}
    \normalfont
    There exist constants $\lambda>0$, $\kappa\ge 0$, and $r\in[0,1)$ such that, for any $\bar\tau$ on the line segment between $\tau$ and $\tau_0$,
    \begin{align}\label{eq:strong-convexity}
        \nabla^2_{\tau}L(\bar{\tau}, \eta)[\tau-\tau_0, \tau-\tau_0]
        \;\ge\; \lambda \| \tau - \tau_0 \|^2_{\mathcal{T}}
        \;-\; \kappa \| \eta -\eta_0 \|^{\tfrac{4}{1+r}}_{\mathcal{H}}.
    \end{align}
    \emph{Rationale:} The risk function $L(\tau,\eta)$ is in a bowl-shape over $\tau$. The $\kappa$ term allows mild curvature deterioration when $\eta\neq\eta_0$; the exponent $4/(1+r)$ is chosen to balance mixed terms via Young’s inequality later.
\end{assumption}

\subsection{Assumption checks for the examples}
We verify that the R- and DR-losses satisfy the above, clarifying how positivity yields curvature and how residualization/DR construction yields orthogonality.

\subsubsection{R-Learner: assumptions hold}
\paragraph{First-order optimality.}
With $\tilde Y\triangleq Y-m_0(X)$, $\tilde T\triangleq T-\pi_0(X)$,
\begin{align}
  \nabla_{\tau}L_{\mathrm{R}}(\tau_0,\eta_0)[h_\tau]
  = -2\,\mathbb{E}\!\left[(\tilde Y-\tilde T\tau_0)\,\tilde T\,h_\tau(X)\right]
  = -2\,\mathbb{E}\!\left[\mathbb E[\epsilon_Y\mid T,X]\tilde T\,h_\tau(X)\right] = 0.
\end{align}
Hence Assumption~\ref{assumption:first-order-optimality} holds.

\paragraph{Strong convexity.}
We have
\begin{align}
     \nabla^2_{\tau}L(\bar{\tau}, \eta)[\tau-\tau_0, \tau-\tau_0] = 2\mathbb E\!\left[ \{\tau(X) - \tau_0(X)\}^2 \,\{T-\pi(X)\}^2\right],
\end{align}
where
\begin{align}
    \mathbb E\!\big[(T-\pi(X))^2\mid X\big]
= \underbrace{\mathrm{Var}(T\mid X)}_{=\pi_0(X)(1-\pi_0(X))}
\;+\; \big(\pi_0(X)-\pi(X)\big)^2
\;\ge\; \pi_0(X)(1-\pi_0(X)).
\end{align}
Therefore,
\begin{align}
    \nabla^2_{\tau}L(\bar{\tau}, \eta)[\tau-\tau_0, \tau-\tau_0] &= 2\mathbb E\!\left[ \{\tau(X) - \tau_0(X)\}^2 \,\{T-\pi(X)\}^2\right] \\
    &\geq 2\mathbb E\!\left[ \{\tau(X) - \tau_0(X)\}^2 \mathrm{Var}(T\mid X) \right]  \\
    &= 2\mathbb E\!\left[ \{\tau(X) - \tau_0(X)\}^2 \pi_0(X) \{1-\pi_0(X)\}\right]  \\
    &\geq 2E\!\left[ \{\tau(X) - \tau_0(X)\}^2 c \{1-c\}\right] \\
    &=  2c(1-c) \| \tau - \tau_0 \|^2_{2}.
\end{align}
Hence, Assumption~\ref{assumption:strong-convexity} holds with $\lambda=2c(1-c)$ and $\kappa=0$ (taking $\|\cdot\|_{\mathcal T}=\|\cdot\|_2$).

\subsubsection{DR-Learner: assumptions hold}
\paragraph{First-order optimality.}
\begin{align}
    \nabla_{\tau}L_{\mathrm{DR}}(\tau_0,\eta_0)[h_\tau]
    = -2\,\mathbb E\!\left[\{\varphi(V;\eta_0)-\tau_0(X)\}\,h_\tau(X)\right]
    = 0,
\end{align}
since $\mathbb E[\varphi(V;\eta_0)\mid X]=\tau_0(X)$.

\paragraph{Strong convexity.}
Since $L_{\mathrm{DR}}(\tau,\eta)=\mathbb{E}[\{\varphi(V;\eta)-\tau(X)\}^2]$, direct differentiation gives
\begin{align}
    \nabla^2_{\tau}L_{\mathrm{DR}}(\tau,\eta)[\tau-\tau_0,\tau-\tau_0]  = 2 \| \tau-\tau_0\|^2_{2},
\end{align}
which shows that $\kappa = 0$ and $\lambda = 2$ (taking $\|\cdot\|_{\mathcal T}=\|\cdot\|_2$).

\subsection{Main Result}

\begin{theorem}[\textbf{Fast Rate Convergence}]\label{thm:fast-rate-convergence-1}
    \normalfont
    Suppose Assumption~\ref{assumption:first-order-optimality} and \ref{assumption:strong-convexity} hold. Then,
    \begin{align}
        \| \hat\tau - \tau_0 \|^2_{\mathcal{T}} \leq \tfrac{2}{\lambda} R_{\mathcal{T}}(\hat\tau; \hat\eta, \epsilon)  + \tfrac{2}{\lambda}  \{ \nabla_{\tau} L(\tau_0,\eta_0)[\hat\tau - \tau_0] -  \nabla_{\tau}L(\tau_0, \hat\eta)[\hat\tau - \tau_0]\} + \tfrac{\kappa}{\lambda}\| \hat\eta - \eta_0 \|^{\tfrac{4}{1+r}}_{\mathcal{H}}.
    \end{align}
\end{theorem}
\begin{proof}[\textbf{Proof of Thm.~\ref{thm:fast-rate-convergence-1}}]
    By applying the Taylor's expansion and rearranging, we have
    \begin{align}
        \tfrac{1}{2}\nabla^2_{\tau}L(\bar\tau,\hat\eta)[\hat\tau-\tau_0,\hat\tau-\tau_0] = L(\hat\tau, \hat\eta) - L(\tau_0,\hat\eta) - \nabla_{\tau}L(\tau_0, \hat\eta)[\hat\tau - \tau_0],\nonumber
    \end{align}
    where $\bar\tau$ is on the line segment between $\hat\tau$ and $\tau_0$.

    Using Assumption~\ref{assumption:strong-convexity}, we have
    \begin{align}
        \frac{\lambda}{2} \| \hat\tau - \tau_0 \|^2_{\mathcal{T}} \leq  L(\hat\tau, \hat\eta) - L(\tau_0,\hat\eta) - \nabla_{\tau}L(\tau_0, \hat\eta)[\hat\tau - \tau_0]+  \frac{\kappa}{2} \| \hat\eta - \eta_0 \|^{\tfrac{4}{1+r}}_{\mathcal{H}}. \nonumber
    \end{align}
    Since $\tau^*_{\hat\eta}$ minimizes $L(\cdot,\hat\eta)$, the target-class statistical bound gives
    \begin{align}
        L(\hat\tau,\hat\eta)-L(\tau_0,\hat\eta)
        \le
        L(\hat\tau,\hat\eta)-L(\tau^*_{\hat\eta},\hat\eta)
        \le
        R_{\mathcal{T}}(\hat\tau;\hat\eta,\epsilon).
    \end{align}
    Since $\nabla_{\tau} L(\tau_0,\eta_0)[\hat\tau - \tau_0] \geq 0$ by Assumption~\ref{assumption:first-order-optimality}, we have
    \begin{align}\label{eq:proof-thm-1}
        \frac{\lambda}{2} \| \hat\tau - \tau_0 \|^2_{\mathcal{T}} \leq R_{\mathcal{T}}(\hat\tau; \hat\eta, \epsilon) + \{ \nabla_{\tau} L(\tau_0,\eta_0)[\hat\tau - \tau_0] -  \nabla_{\tau}L(\tau_0, \hat\eta)[\hat\tau - \tau_0]\} +\frac{\kappa}{2} \| \hat\eta - \eta_0 \|^{\tfrac{4}{1+r}}_{\mathcal{H}}.
    \end{align}
\end{proof}

The middle difference
\begin{align}
    \big\{\nabla_\tau L(\tau_0,\eta_0)-\nabla_\tau L(\tau_0,\hat\eta)\big\}[\hat\tau-\tau_0] \label{eq:nuisance-leakage}
\end{align}
is the \emph{nuisance leakage of the first-order optimality condition}. It is the main channel through which nuisance error affects the target. Under Neyman orthogonality, the leakage is \emph{higher than first order} in $\|\hat\eta-\eta_0\|$ (typically quadratic or a product of nuisance errors), so $\hat\tau$ inherits only a higher-order remainder rather than linear bias.
In particular, for the standard back-door DR-learner it factors into a product of nuisance errors (the usual \emph{double robustness} interface), whereas for the R-learner it enables \emph{fast rates} once the nuisance remainders are sufficiently small. The front-door learners in the main text use the more explicit product-error and stage-error conditions stated in Thms.~\ref{thm:error-fd-dr} and~\ref{thm:error-fd-r}.


\subsubsection{Nuisance Leakage: R-learner}

\begin{theorem}[\textbf{Error Analysis: R-learner}]\label{thm:error-r-learner}
    \normalfont
    Suppose Assumption~\ref{assumption:first-order-optimality} and \ref{assumption:strong-convexity} hold with $\|\cdot\|_{\mathcal{T}} = \|\cdot\|_2$. Suppose $c \leq \pi_0(X) \leq 1-c$ and $\|\tau_0\|_{\infty}<\infty$. Then, with probability $1-\epsilon$,
    \begin{align}
        \|  \hat{\tau} - \tau_0\|_{2}^2
        \lesssim
        R_{\mathcal{T}}(\hat\tau; \hat\eta, \epsilon)
        + \| \hat{\pi} - \pi_0 \|^4_4
        + \| \hat{m} - m_0\|_4^2\| \hat{\pi} - \pi_0\|_4^2,
    \end{align}
    where the hidden constant depends on $c$ and $\|\tau_0\|_\infty$.
\end{theorem}
\begin{proof}[\textbf{Proof of Thm.~\ref{thm:error-r-learner}}]

Let $h_{\tau}(X) \triangleq \hat{\tau}(X) - \tau_0(X)$. Let $\delta_{m}(X)\triangleq \widehat m(X)-m_0(X)$ and $\delta_\pi(X) \triangleq \widehat\pi(X) - \pi_0(X)$.

For a generic nuisance $\eta=\{m,\pi\}$, the first-order risk derivative is
\begin{align}
    \nabla_{\tau}L_{\mathrm{R}}[\tau, \eta](h_\tau) = -2\mathbb{E}[\{Y - m(X) - \tau(X)(T - \pi(X) ) \} \cdot \{T - \pi(X)\} \cdot h_\tau(X)],
\end{align}
We note that $\nabla_{\tau}L_{\mathrm{R}}[\tau_0, \eta_0](h_\tau) = 0$, as shown in the first-order optimality condition analysis. To analyze the leakage, we rewrite a few terms here:
\begin{align}
    Y - \widehat m - \tau_0(T-\widehat\pi) &= \underbrace{ Y-m_0-\tau_0(T - \pi_0) }_{\epsilon_Y} - \delta_m + \tau_0\delta_\pi,  \\
    T - \widehat\pi &= T-\pi_0 - \delta_\pi.
\end{align}
Then, we can rewrite the first-order risk as follows:
\begin{align}
    \nabla_{\tau}L_{\mathrm{R}}[\tau_0, \hat\eta](h_\tau) &= -2\mathbb{E}[\{\epsilon_Y - \delta_m+\tau_0\delta_\pi \} \cdot (T-\pi_0-\delta_\pi) \cdot h_{\tau}] \\
    &= 2\mathbb{E}[ \{\tau_0\delta_\pi^2 - \delta_m \delta_\pi \} h_{\tau}] \\
    &\leq 2 \vert \mathbb{E}[ \tau_0\delta_\pi^2 h_{\tau} ] \vert + 2 \vert \mathbb{E}[\delta_m \delta_\pi h_{\tau}] \vert  \\
    &\leq 2 \| \delta_\pi \|^2_{4} \cdot \|\tau_0 \|_{{\infty}} \cdot \| h_{\tau} \|_{2} + 2 \| \delta_m \|_{4} \cdot \|\delta_\pi \|_{4} \cdot \|h_{\tau} \|_{2} \\
    &= 2\| h_{\tau} \|_{2} \cdot \left(  \|\tau_0 \|_{{\infty}} \|\delta_{\pi} \|^2_{4} + \|\delta_m \|_{4} \cdot \|\delta_\pi \|_{4}  \right).
\end{align}
Then, for any $\alpha > 0$, Young's inequality (with $p=q=2$) gives
\begin{align}
    & \nabla_{\tau} L(\tau_0,\eta_0)[\hat\tau - \tau_0] -  \nabla_{\tau}L(\tau_0, \hat\eta)[\hat\tau - \tau_0]  \\
    &\leq 2\| h_{\tau} \|_{2} \cdot \left(  \|\tau_0 \|_{\infty} \|\delta_{\pi} \|^2_{4} + \|\delta_m \|_{4} \cdot \|\delta_\pi \|_{4}  \right) \\
    &\leq \alpha \|h_{\tau} \|_{2}^2 + \frac{1}{\alpha} \left(  \|\tau_0 \|_{\infty} \|\delta_{\pi} \|^2_{4} + \|\delta_m \|_{4} \cdot \|\delta_\pi \|_{4}  \right)^2 \\
    &= \alpha \| h_{\tau} \|^2_{2} + \frac{2}{\alpha} \|\tau_0\|^2_{\infty} \|\delta_\pi \|^4_4 + \frac{2}{\alpha}\|\delta_m\|^{2}_{4} \| \delta_{\pi}\|^2_{4}.
\end{align}
Choose $\alpha = \lambda/4$, where $\lambda=2c(1-c)$. Let $R_\star \triangleq R_{\mathcal{T}}(\hat{\tau};\hat{\eta},\epsilon)$. Then, by Thm.~\ref{thm:fast-rate-convergence-1}, we have
\begin{align}
    & \|h_{\tau}\|^2_2 \leq \frac{2}{\lambda}R_\star + \frac{2}{\lambda}\frac{\lambda}{4} \|h_{\tau} \|^2_2 + \frac{16\|\tau_0\|_\infty^2}{\lambda^2} \|\delta_{\pi}\|^4_4 + \frac{16}{\lambda^2} \|\delta_m \|^2_4 \| \delta_\pi\|^2_4, \\
    & \implies \frac{1}{2} \|h_{\tau}\|^2_2 \leq \frac{2}{\lambda}R_\star + \frac{16\|\tau_0\|_\infty^2}{\lambda^2} \|\delta_{\pi}\|^4_4 + \frac{16}{\lambda^2} \|\delta_m \|^2_4 \| \delta_\pi\|^2_4,
\end{align}
which completes the proof.

\end{proof}

\subsubsection{Nuisance Leakage: DR-learner}

\begin{theorem}[\textbf{Error Analysis: DR-learner}]\label{thm:error-dr-learner}
    \normalfont
    Suppose Assumption~\ref{assumption:first-order-optimality} and \ref{assumption:strong-convexity} hold with $\|\cdot\|_{\mathcal{T}} = \|\cdot\|_2$, under bounded-overlap and moment conditions making the displayed norms finite. Then, with probability $1-\epsilon$,
    \begin{align}
        \|  \hat{\tau} - \tau_0\|_{2}^2
        \lesssim
        R_{\mathcal{T}}(\hat\tau; \hat\eta, \epsilon)
        + \| \widehat\omega-\omega_0 \|^2_{4} \| \widehat\mu - \mu_0 \|^2_4.
    \end{align}
\end{theorem}
\begin{proof}[\textbf{Proof of Thm.~\ref{thm:error-dr-learner}}]
    Let $h_{\tau}(X) \triangleq \hat{\tau}(X) - \tau_0(X)$. Let $\delta_{\mu}(T,X)\triangleq \widehat\mu(T,X)-\mu_0(T,X)$ and $\delta_\omega(T,X) \triangleq \widehat\omega(T,X)-\omega_0(T,X)$.

    Note
    \begin{align}
        \nabla_{\tau}L_{\mathrm{DR}}(\tau_0,\hat\eta)[h_{\tau}] = -2\mathbb{E}[\{\varphi(V;\hat\eta) - \tau_0(X)\} h_{\tau}(X)].
    \end{align}
    We note $\nabla_{\tau}L_{\mathrm{DR}}(\tau_0,\eta_0)[h_{\tau}] = 0$, as shown in the first-order optimality condition analysis. Also,
    \begin{align}
        &\vert \mathbb{E}[\{\varphi(V;\hat\eta) - \tau_0(X)\} h_{\tau}(X)] \vert \\
        &= \vert \mathbb{E}[\delta_\omega(T,X)\{\mu_0(T,X)-\widehat\mu(T,X)\} h_{\tau}(X)] \vert  \\
        &\leq \| h_{\tau} \|_2 \| \delta_\omega \delta_\mu \|_2 \\
        &\leq \| h_{\tau} \|_2 \| \widehat\omega-\omega_0 \|_{4} \| \widehat\mu - \mu_0 \|_4.
    \end{align}
    The displayed equality follows from the standard DR score identity, using $\mathbb{E}[Y-\mu_0(T,X)\mid T,X]=0$ and $\mathbb{E}[\varphi(V;\eta_0)\mid X]=\tau_0(X)$.
    Then, for any $\alpha > 0$, Young's inequality (with $p=q=2$) gives
    \begin{align}
        &2 \vert\mathbb{E}[\{\varphi(V;\hat\eta) - \tau_0(X)\} h_{\tau}(X)]\vert  \\
        &\leq \alpha\| h_{\tau}\|^{2}_{2} + \frac{1}{\alpha} \| \widehat\omega-\omega_0 \|^2_{4} \| \widehat\mu - \mu_0 \|^2_4.
    \end{align}

    Choose $\alpha = \lambda/4$. Let $R_\star \triangleq R_{\mathcal{T}}(\hat{\tau};\hat{\eta},\epsilon)$. Then, by Thm.~\ref{thm:fast-rate-convergence-1}, we have
    \begin{align}
        & \|h_{\tau}\|^2_2 \leq \frac{2}{\lambda}R_\star + \frac{2}{\lambda}\frac{\lambda}{4} \|h_{\tau} \|^2_2 + \frac{16}{\lambda^2} \| \widehat\omega-\omega_0 \|^2_{4} \| \widehat\mu - \mu_0 \|^2_4 \\
        & \implies \frac{1}{2} \|h_{\tau}\|^2_2 \leq \frac{2}{\lambda}R_\star + \frac{16}{\lambda^2}\| \widehat\omega-\omega_0 \|^2_{4} \| \widehat\mu - \mu_0 \|^2_4,
    \end{align}
    which completes the proof.
\end{proof}

\newpage
\section{Proofs}\label{sec:proofs}

\subsection{Proof of the Conditional Front-Door Identification Formula}
Under the conditional front-door criterion, the generalized front-door adjustment holds conditional on $C$:
\begin{align}
    \tau_{\bar{x}}(C)
    =
    \mathbb{E}[Y\mid \mathrm{do}(X=\bar{x}),C]
    =
    \sum_{z}q(z\mid \bar{x}C)\sum_x m(zxC)e(x\mid C).
\end{align}
Subtracting the formula for $\bar{x}=0$ from the formula for $\bar{x}=1$ yields
\begin{align}
    \tau(C)
    =
    \sum_{z,x}\{q(z\mid 1C)-q(z\mid 0C)\}e(x\mid C)m(zxC),
\end{align}
which is Eq.~\eqref{eq:target-estimator}.

\subsection{Proof of Lemma~\ref{lemma:fdpo-analysis}}
We prove the conditional mean identity for each $\bar{x}\in\{0,1\}$ at the true front-door nuisance $\eta_0=\{m,e,q\}$. First,
\begin{align}
    \mathbb{E}\!\left[\xi_{\bar{x}}(ZXC)\{Y-m(ZXC)\}\mid C\right]=0,
\end{align}
by iterated expectation conditional on $(Z,X,C)$. Second,
\begin{align}
    \mathbb{E}\!\left[r_{me}(ZC)\mid X,C\right]
    =
    \sum_z r_{me}(zC)q(z\mid XC)
    =
    \nu_{meq}(XC),
\end{align}
so another application of iterated expectation gives
\begin{align}
    \mathbb{E}\!\left[\pi_{\bar{x}}(XC)\{r_{me}(ZC)-\nu_{meq}(XC)\}\mid C\right]=0.
\end{align}
Finally,
\begin{align}
    \mathbb{E}\!\left[s_{mq_{\bar{x}}}(XC)\mid C\right]
    &=
    \sum_x e(x\mid C)\sum_z m(zxC)q(z\mid \bar{x}C)\\
    &=
    \tau_{\bar{x}}(C),
\end{align}
by the conditional front-door formula. Summing the three displayed terms proves
\begin{align}
    \mathbb{E}\!\left[\varphi_{\bar{x}}(V;\eta_0)\mid C\right]=\tau_{\bar{x}}(C).
\end{align}

\subsection{Proof of Lemma~\ref{lemma:fdpo-bias}}
By definition,
\begin{align}
    B_{\bar{x}}(C)=\mathbb{E}\!\left[\varphi_{\bar{x}}(V;\widehat{\eta})\mid C\right]-\tau_{\bar{x}}(C).
\end{align}
Thus, for $j\in\{0,1\}$,
\begin{align}
    B_j(C)=\mathbb{E}\!\left[\varphi_j(V;\widehat{\eta})\mid C\right]-\tau_j(C).
\end{align}
Therefore,
\begin{align}
    \mathbb{E}\!\left[D_{\widehat{\eta}}(V)\mid C\right]-\tau(C)
    &=
    \mathbb{E}\!\left[\varphi_1(V;\widehat{\eta})-\varphi_0(V;\widehat{\eta})\mid C\right]
    -\{\tau_1(C)-\tau_0(C)\}\\
    &=
    B_1(C)-B_0(C).
\end{align}

\subsection{Proof of Theorem~\ref{thm:error-fd-dr}}
By Lemma~\ref{lemma:fdpo-bias},
\begin{align}
    \mu_{\widehat{\eta}}(C)-\tau(C)=B_1(C)-B_0(C).
\end{align}
Hence
\begin{align}
    \widehat{\tau}_{\mathrm{DR}}(C)-\tau(C)
    =
    \{\widehat{\tau}_{\mathrm{DR}}(C)-\mu_{\widehat{\eta}}(C)\}
    +
    \{B_1(C)-B_0(C)\}.
\end{align}
Using $(a+b)^2\le 2a^2+2b^2$ and integrating over $C$,
\begin{align}
    \|\widehat{\tau}_{\mathrm{DR}}-\tau\|_2^2
    \le
    2\|\widehat{\tau}_{\mathrm{DR}}-\mu_{\widehat{\eta}}\|_2^2
    +
    2\|B_1-B_0\|_2^2.
\end{align}
By the definition of $R_{\mathrm{DR}}$, the first term equals $2R_{\mathrm{DR}}$, which gives the result.

\subsection{Proof of Corollary~\ref{cor:fd-dr-quasi-oracle}}
Theorem~\ref{thm:error-fd-dr} shows that the target-regression error and the squared conditional-bias remainder are the only two terms in the bound. If $R_{\mathrm{DR}}=O_p(r_{\mathrm{DR},n})$ and $\|B_1-B_0\|_2^2=o_p(r_{\mathrm{DR},n})$, the target-regression term determines the first-order rate. If the conditional bias is second order in nuisance errors and the nuisance factors in those products are $O_p(n^{-1/4})$, the squared product scale is $O_p(n^{-1})$. The stated quasi-oracle interpretation follows only when that remainder does not dominate the target oracle rate.

\subsection{Proof of Proposition~\ref{prop:ple-FD}}
Define
\begin{align}
    \epsilon_X &\triangleq X-e_X(C),\\
    \epsilon_Z &\triangleq Z-a(C)-Xb(C),\\
    \epsilon_Y &\triangleq Y-f(XC)-Zg(XC).
\end{align}
The first residual satisfies $\mathbb{E}[\epsilon_X\mid C]=0$ by the definition of $e_X(C)$. For $Z$, when $X=0$,
\begin{align}
    \mathbb{E}[\epsilon_Z\mid X=0,C]
    =
    \mathbb{E}[Z\mid X=0,C]-a(C)=0.
\end{align}
When $X=1$,
\begin{align}
    \mathbb{E}[\epsilon_Z\mid X=1,C]
    =
    \mathbb{E}[Z\mid X=1,C]-a(C)-b(C)=0.
\end{align}
Thus $\mathbb{E}[\epsilon_Z\mid X,C]=0$. For $Y$, when $Z=0$,
\begin{align}
    \mathbb{E}[\epsilon_Y\mid Z=0,X,C]
    =
    \mathbb{E}[Y\mid Z=0,X,C]-f(XC)=0.
\end{align}
When $Z=1$,
\begin{align}
    \mathbb{E}[\epsilon_Y\mid Z=1,X,C]
    =
    \mathbb{E}[Y\mid Z=1,X,C]-f(XC)-g(XC)=0.
\end{align}
Therefore $\mathbb{E}[\epsilon_Y\mid Z,X,C]=0$.

\subsection{Proof of Theorem~\ref{thm:fd-hte-ple}}
Since $Z$ is binary,
\begin{align}
    m(zxC)=m(0xC)+z\{m(1xC)-m(0xC)\}
    =
    f(xC)+zg(xC).
\end{align}
Substituting this into the conditional front-door formula gives
\begin{align}
    \tau_{\bar{x}}(C)
    &=
    \sum_z q(z\mid \bar{x}C)\sum_x e(x\mid C)\{f(xC)+zg(xC)\}\\
    &=
    \sum_x e(x\mid C)f(xC)
    +
    \left\{\sum_z zq(z\mid \bar{x}C)\right\}
    \sum_x e(x\mid C)g(xC)\\
    &=
    \sum_x e(x\mid C)f(xC)
    +
    \mathbb{E}[Z\mid X=\bar{x},C]\gamma_g(C).
\end{align}
The first term does not depend on $\bar{x}$. Taking the difference between $\bar{x}=1$ and $\bar{x}=0$ yields
\begin{align}
    \tau(C)
    =
    \{\mathbb{E}[Z\mid X=1,C]-\mathbb{E}[Z\mid X=0,C]\}\gamma_g(C)
    =
    b(C)\gamma_g(C).
\end{align}

\subsection{Proof of Lemma~\ref{lemma:property-pseudo-g}}
Using $\mathbb{E}[X-\widehat e_X(C)\mid C]=e_X(C)-\widehat e_X(C)$,
\begin{align}
    \mathbb{E}[\zeta_{\widehat{\eta}_z}(XC)\mid C]
    &=
    \{1-\widehat e_X(C)\}\widehat g(0C)
    +\widehat e_X(C)\widehat g(1C)\\
    &\quad+
    \{e_X(C)-\widehat e_X(C)\}\{\widehat g(1C)-\widehat g(0C)\}\\
    &=
    e_X(C)\widehat g(1C)+\{1-e_X(C)\}\widehat g(0C).
\end{align}
Subtracting
\begin{align}
    \gamma_g(C)=e_X(C)g(1C)+\{1-e_X(C)\}g(0C)
\end{align}
gives
\begin{align}
    \mathbb{E}[\zeta_{\widehat{\eta}_{z}}(XC)\mid C]-\gamma_g(C)
    =
    e_X(C)\{\widehat g(1C)-g(1C)\}
    +\{1-e_X(C)\}\{\widehat g(0C)-g(0C)\}.
\end{align}

\subsection{Proof of Theorem~\ref{thm:error-fd-r}}
By Theorem~\ref{thm:fd-hte-ple},
\begin{align}
    \widehat{\tau}_{\mathrm{R}}-\tau
    =
    \widehat b\widehat\gamma-b\gamma_g
    =
    \widehat b(\widehat\gamma-\gamma_g)+\gamma_g(\widehat b-b).
\end{align}
Using $(a+b)^2\le 2a^2+2b^2$,
\begin{align}
    \|\widehat{\tau}_{\mathrm{R}}-\tau\|_2^2
    &\le
    2\|\widehat b(\widehat\gamma-\gamma_g)\|_2^2
    +
    2\|\gamma_g(\widehat b-b)\|_2^2\\
    &\le
    2M_b^2\|\widehat\gamma-\gamma_g\|_{2,C}^2
    +2M_\gamma^2\|\widehat b-b\|_{2,C}^2.
\end{align}
The oracle-stage definitions give
\begin{align}
    \|\widehat b-b\|_{2,C}^2
    &=
    \|(\widehat b-\bar b)+(\bar b-b)\|_{2,C}^2\\
    &\le
    2N_b+2R_b.
\end{align}
Similarly, writing $\gamma_{\bar g}(C)=\mathbb{E}[\bar g(X,C)\mid C]$,
\begin{align}
    \|\widehat\gamma-\gamma_g\|_{2,C}^2
    &\le
    3\|\widehat\gamma-\bar\gamma\|_{2,C}^2
    +3\|\bar\gamma-\gamma_{\bar g}\|_{2,C}^2
    +3\|\gamma_{\bar g}-\gamma_g\|_{2,C}^2\\
    &\le
    3N_\gamma+3R_\gamma+3R_g.
\end{align}
The last step uses Jensen's inequality:
\begin{align}
    \|\gamma_{\bar g}-\gamma_g\|_{2,C}^2
    =
    \|\mathbb{E}\{\bar g(X,C)-g(X,C)\mid C\}\|_{2,C}^2
    \le
    \|\bar g-g\|_{2,XC}^2
    =
    R_g.
\end{align}
Substitution gives
\begin{align}
    \|\widehat\tau_{\mathrm{R}}-\tau\|_2^2
    \lesssim
    R_b+R_g+R_\gamma+N_b+N_g+N_\gamma.
\end{align}

\subsection{Proof of Corollary~\ref{cor:fd-r-quasi-oracle}}
Theorem~\ref{thm:error-fd-r} reduces the FD-R error to stage oracle errors and nuisance remainders. If $R_b+R_g+R_\gamma=O_p(r_n)$ and $N_b+N_g+N_\gamma=o_p(r_n)$, then the nuisance part is lower order and the oracle stage terms determine the first-order rate. The terms $N_b$ and $N_g$ are the BD-R nuisance leakages for the two back-door subproblems; under the standard BD-R interface they are fourth-order or second-order product terms, so $O_p(n^{-1})$ when the corresponding nuisance factors are $O_p(n^{-1/4})$. The stated quasi-oracle or fast-rate interpretation therefore also requires the final pseudo-$g$ regression leakage $N_\gamma$ to be lower order than the oracle stage rate.

\newpage
\section{Simulation Details}\label{sec:sim-details}

The full information on implementation is available at \url{https://github.com/yonghanjung/FD-CATE}.

\subsection{Synthetic simulation}\label{subsec:synthetic-sim}

\paragraph{Data-generating process (DGP).}
For a given sample size $n$ and dimension $d$,
\begin{align*}
C &\sim \mathcal{N}(0,I_d),\qquad U\sim\mathcal{N}(0,1),\\
\Pr(X=1\mid C,U) &= \sigma\!\big(\beta_0+\beta_c^\top C+\beta_u U\big),\\
\Pr(Z=1\mid X,C) &= \sigma\!\big(\alpha_0+\alpha_c^\top C+\alpha_x X\big),\\
Y &= \theta_0+\theta_c^\top C+\theta_z Z+\theta_u U+\varepsilon,\qquad \varepsilon\sim\mathcal{N}(0,1).
\end{align*}
We choose “moderate-positivity’’ coefficients to avoid extreme propensities:
\[
\beta_0=0.1,\ \beta_c=0.7\,w_X,\ \beta_u=0.7;\quad
\alpha_0=0.1,\ \alpha_c=0.7\,w_Z,\ \alpha_x=1.2;\quad
\theta_0=0,\ \theta_c=0.7\,w_Y,\ \theta_z=1.4,\ \theta_u=-2.4,
\]
where $w_X,w_Z,w_Y\sim\mathcal{N}(0,I_d)$ are $\ell_2$-normalized. The ground-truth heterogeneous effect is
\[
\tau_{\text{true}}(c)=\theta_z\Big\{\sigma\big(\alpha_0+\alpha_c^\top c+\alpha_x\big)-\sigma\big(\alpha_0+\alpha_c^\top c\big)\Big\}.
\]

\paragraph{Learning algorithms and cross-fitting.}
All nuisances use XGBoost \citep{chen2016xgboost} with the same configuration: $\mathrm{n\_estimators}=50$, $\mathrm{max\_depth}=3$, learning rate $=0.1$, subsample $=0.9$, $\mathrm{colsample_bytree} =0.9$, $\ell_2$ penalty $\lambda=1.0$, histogram tree method. We employ two-fold cross-fitting for FD-PI/FD-DR. FD-R uses a three-way split: nuisances on $\mathcal{D}_1$, $(b,g)$ via BD-R on $\mathcal{D}_2$, and the final $\gamma_g$ regression on $\mathcal{D}_3$, then swap/average. Final regressions on pseudo-outcomes (FD-DR) or on the pseudo-$g$ target (FD-R) use ridge OLS with $\alpha=10^{-6}$. To control variance, \emph{only denominators} appearing in inverse weights/density ratios are floored at $0.05$; numerators are never clipped.

\paragraph{Structural nuisance noise.}
To stress robustness empirically, we inject \emph{simulation-level} nuisance perturbations at the $n^{-1/4}$ scale:
\[
p\in\{e,q\}:\quad p\mapsto \mathrm{clip}_{(0,1)}\!\big(p+\delta\,\varepsilon\big),\quad
\mu\in\{m_Y,m\}:\quad \mu\mapsto \mu+\delta\,\varepsilon,\qquad
\varepsilon\sim\mathcal{N}\!\big(n^{-1/4},\,n^{-1/4}\big),
\]
sweeping the knob $\delta\in\{0,0.2,0.4,0.6,0.8,1.0\}$. Denominators used in weights are \emph{frozen before noise} and floored.

\paragraph{Weak-overlap stress test.}
We steepen the treatment logit via
\[
\Pr(X=1\mid C,U)=\sigma\{\beta_0+\kappa_e(\beta_c^\top C+\beta_u U)\},\qquad \kappa_e\in\{2,4,6,8,10\},
\]
keeping the mediator regression moderate. This inflates inverse-weight variance (affecting FD-PI/FD-DR) while FD-R remains variance-friendly (no density ratios).

\paragraph{Design grid and reporting.}
Unless noted, $d=10$. We vary $n\in\{1{,}000,2{,}500,5{,}000,10{,}000,20{,}000,50{,}000\}$ and use $R=100$ Monte Carlo replications. We report RMSE
\[
\mathrm{RMSE}=\Big(\mathbb{E}\big[(\widehat\tau(C)-\tau_{\text{true}}(C))^2\big]\Big)^{1/2}
\]
with mean~$\pm$~95\% normal CIs across replications.

\subsection{Real-world study: State seat-belt laws and fatalities (FARS)}
\label{subsec:realworld-fars}

\paragraph{Setting and estimand.}
We study the effect of adopting a \emph{primary seat-belt law} on motor-vehicle occupant fatality rates using a state--year panel constructed from the National Highway Traffic Safety Administration’s Fatality Analysis Reporting System (FARS) and companion NHTSA survey tables \citep{nhtsa_fars2000}. Let $C$ denote observed state--year covariates, $X\in\{0,1\}$ indicate whether a primary law is in force, $Z\in \{0,1\}$ be the observed seat-belt use, and $Y$ the \emph{occupant fatality}. Our target is the conditional front–door effect $\tau(C)$.

The FD assumptions are plausible here because (i) the causal pathway from $X$ to $Y$ operates via increased belt use; and (ii) rich $C$ may be sufficient to explain spurious paths between $X$ and $Z$; and $Z$ and $Y$; and (iii) \emph{positivity} holds empirically since belt use is neither zero nor universal in either law regime.

Our result is described in Fig.~\ref{fig:real}.

\paragraph{Data and preprocessing.}
Following the analysis script (\texttt{analyze\_fars\_2000\_fd.py}) provided in supplements, we build a balanced state--year panel and construct $(C,X,Z,Y)$ as follows.
\begin{itemize}[leftmargin=*]
\item \textbf{Treatment $X$:} indicator that a \emph{primary} seat-belt law is active in a given state and year.
\item \textbf{Mediator $Z$:} \emph{belt-use} from NHTSA surveys ($\{0,1\}$).
\item \textbf{Outcome $Y$:} \emph{occupant fatality}.
\item \textbf{Covariates $C$:} state and year fixed effects and policy-relevant controls compiled at the state--year level: coarse weather severity, and road-class mix and drivers' status.
\end{itemize}

\paragraph{Estimators and learning protocol.}
We fit the plug-in FD baseline (FD-PI), \textbf{FD-DR-Learner}, and \textbf{FD-R-Learner} exactly as in the synthetic study: all nuisances are learned with XGBoost (XGBoost; $50$ trees, depth $3$, learning rate $0.1$, subsample/colsample $0.9$, $\ell_2$ penalty $\lambda{=}1$); FD-PI/FD-DR use two-fold cross-fitting, and FD-R uses a three-way split (nuisances $\rightarrow$ $(b,g)$ via BD-R $\rightarrow$ $\gamma_g$ via pseudo-$g$), followed by swap-averaging. The final regression on pseudo-outcomes (FD-DR) or on pseudo-$g$ targets (FD-R) uses ridge OLS. To stabilize finite-sample variance we \emph{floor only denominators} that appear in inverse weights/density ratios at $0.05$; numerators are never clipped.

\endgroup

\end{document}